\documentclass{article}

\PassOptionsToPackage{numbers, compress}{natbib}

    \usepackage[preprint]{neurips_2021}

\usepackage[utf8]{inputenc} 
\usepackage[T1]{fontenc}    
\usepackage{hyperref}       
\usepackage{url}            
\usepackage{booktabs}       
\usepackage{amsfonts}       
\usepackage{nicefrac}       
\usepackage{microtype}      

\usepackage{graphicx}
\usepackage{comment}
\usepackage{amsmath,amssymb,amsthm} 
\usepackage{color}
\PassOptionsToPackage{usenames,dvipsnames}{xcolor}
\usepackage{tikz}
\usepackage{multirow}
\usepackage{bm}
\usepackage{booktabs}
\usepackage[normalem]{ulem}
\useunder{\uline}{\ul}{}
\usepackage{arydshln}   
\usepackage{caption}
\usepackage{subcaption}
\usepackage{algorithm}
\usepackage{listings}
\usepackage{comment}
\newtheorem{thm}{Theorem}   
\usepackage{xspace}

\DeclareFixedFont{\ttb}{T1}{txtt}{bx}{n}{12} 
\DeclareFixedFont{\ttm}{T1}{txtt}{m}{n}{12}  

\definecolor{deepblue}{rgb}{0,0,0.5}
\definecolor{deepred}{rgb}{0.6,0,0}
\definecolor{deepgreen}{rgb}{0,0.5,0}

\definecolor{codegreen}{rgb}{0,0.6,0}
\definecolor{codegray}{rgb}{0.5,0.5,0.5}
\definecolor{codepurple}{rgb}{0.58,0,0.82}
\definecolor{backcolour}{rgb}{0.95,0.95,0.92}

\definecolor{codeblue}{rgb}{0.25,0.5,0.5}
\newcommand{\mixup}{\textit{mixup}}
\newcommand{\nmixup}{$\zeta$-\textit{mixup}}
\newcommand{\pvar}{\gamma}
\newcommand{\pvalmin}{1.72865}
\newcommand{\pvarmin}{\pvar_{\mathrm{min}}}
\newcommand{\rise}[1]{\textcolor{gray}{\textsubscript{$+$#1}}}

\newcommand{\Rise}[1]{\textcolor{Green}{\xspace\textsubscript{\bf $+$#1}}}

\def\ie{i.e.,~}
\def\eg{e.g.,~}
\def\etal{et al.}
\def\fig{Fig. }
\def\eqn{Eqn. }

\title{Multi-Sample \nmixup: Richer, More Realistic Synthetic Samples from a $p$-Series Interpolant}

\author{
  Kumar Abhishek$^1$, Colin J.~Brown$^2$, Ghassan Hamarneh$^1$\\
  $^1$School of Computer Science, Simon Fraser University, Canada\\
  \texttt{\{kabhishe, hamarneh\}@sfu.ca}\\
  $^2$Hinge Health, Canada\\ 
  \texttt{colin.brown@hingehealth.com}
}

\begin{document}

\maketitle

\begin{abstract}
  Modern deep learning training procedures rely on model regularization techniques such as data augmentation methods, which generate training samples that increase the diversity of data and richness of label information. A popular recent method, \mixup, uses convex combinations of pairs of original samples to generate new samples. However, as we show in our experiments, \mixup~can produce undesirable synthetic samples, where the data is sampled off the manifold and can contain incorrect labels. We propose \nmixup, a generalization of \mixup~with provably and demonstrably desirable properties that allows convex combinations of $N \geq 2$ samples, leading to more realistic and diverse outputs that incorporate information from $N$ original samples by using a $p$-series interpolant. We show that, compared to \mixup, \nmixup~better preserves the intrinsic dimensionality of the original datasets, which is a desirable property for training generalizable models. Furthermore, we show that our implementation of \nmixup~is faster than \mixup
, and extensive evaluation on controlled synthetic and 24 real-world natural and medical image classification datasets shows that \nmixup~outperforms \mixup~and traditional data augmentation techniques.
\end{abstract}

\section{Introduction}

Deep learning-based techniques have demonstrated unprecedented performance improvements over the last decade in a wide range of tasks, including but not limited to image classification, segmentation, and detection, speech recognition, natural language processing, and graph processing~\cite{schmidhuber2015deep,lecun2015deep,alom2018history,wu2020comprehensive}. 
These deep neural networks (DNNs) have a large number of parameters, often in the tens to hundreds of millions, and
training accurate, robust, and generalizable models has largely been possible because of large public datasets~\cite{deng2009imagenet,lin2014microsoft,cordts2016cityscapes}, efficient training methods~\cite{rumelhart1986learning,stanley2002evolving}, hardware-accelerated training~\cite{steinkraus2005using,chellapilla2006high,raina2009large,ciresan2010deep}, advances in network architecture design~\cite{simonyan2014very,ronneberger2015u,he2016deep}, advanced optimizers~\cite{duchi2011adaptive,zeiler2012adadelta,kingma2014adam,dozat2016incorporating}, new regularization layers~\cite{srivastava2014dropout,ioffe2015batch}, and other novel regularization techniques. While techniques such as weight decay~\cite{krogh1991simple}, dropout~\cite{srivastava2014dropout}, batch normalization~\cite{ioffe2015batch}, and stochastic depth~\cite{huang2016deep} can be considered as ``data independent" regularization schemes~\cite{guo2019mixup}, popular ``data dependent" regularization approaches include data augmentation~\cite{lecun1998gradient,ciresan2012multi,krizhevsky2012imagenet,zeiler2014visualizing,simonyan2014very} and adversarial training~\cite{goodfellow2014explaining,bai2021recent}.

\begin{figure*}[ht!]
    \centering
    \includegraphics[width=0.8\textwidth]{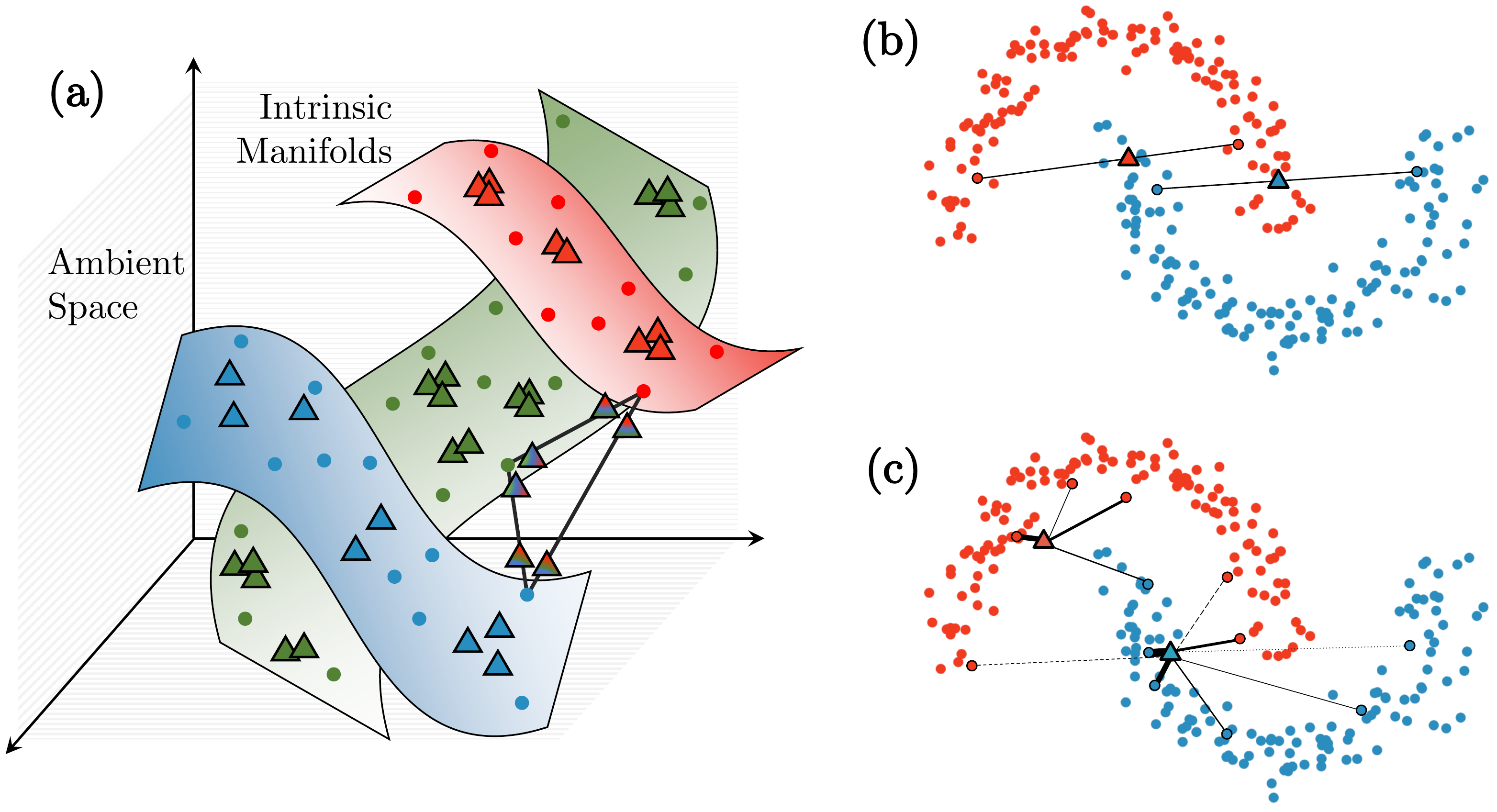}
    \caption{Overview of \mixup~(b) and \nmixup~(a, c). The original and synthesized samples are denoted by $\circ$ and $\bigtriangleup$ respectively, and line segments indicate which original samples were used to create the new ones. The line thicknesses denote the relative weights assigned to original samples. Observe how \nmixup~can mix any number of samples (\eg 3 in (a) and 4 or 8 in (c)), and that \nmixup's formulation allows the generated samples to be close to the original distribution while still incorporating rich information from several samples.}
        \label{fig:overview}
\end{figure*}

Given the large parameter space of deep learning models, training on small datasets tends to cause the models to overfit to the training samples. This is especially a problem when training with data from high dimensional input spaces, such as images, because the sampling density is exponentially proportional to $1/\mathcal{D}$, where $\mathcal{D}$ is the dimensionality of the input space~\cite{hastie2009elements}. As $\mathcal{D}$ grows larger (typically $10^4$ to $10^6$ for most real-world image datasets), we need to increase the number of samples exponentially in order to retain the same sampling density. As a result, it is imperative that the training datasets for these models have a sufficiently large number of samples in order to prevent overfitting. Moreover, deep learning models generally exhibit good generalization performance when evaluated on samples that come from a distribution similar to the training samples' distribution. In addition to their regularization effects to prevent overfitting~\cite{hernandez2018further,hernandez2018data}, data augmentation techniques also help the training by synthesizing more samples in order to better learn the training distributions.

Traditional image data augmentation techniques include geometric- and inte-nsity-based transformations, such as affine transformations, rotation, scaling, zooming, cropping, adding noise, etc., and are quite popular in the deep learning literature. For a comprehensive review of data augmentation techniques for deep learning methods on images, we refer the interested readers to the survey by Shorten \etal~\cite{shorten2019survey}. In this paper, we focus on a recent and popular data augmentation technique based on a rather simple idea, which generates a convex combination of a pair of input samples, variations of which are presented as \mixup~\cite{zhang2018mixup}, Between-Class learning~\cite{tokozume2018between}, and SamplePairing~\cite{inoue2018data}. The most popular of these approaches,~\mixup~\cite{zhang2018mixup}, performs data augmentation by generating new training samples from convex combinations of pairs of original samples and linear interpolations of their corresponding labels, leading to new training samples, which are obtained by essentially overlaying 2 images with different transparencies, and new training labels, which are soft probabilistic labels. Other related augmentation methods can broadly be grouped into 3 categories: (a) methods that crop or mask region(s) of the original input image followed by \mixup~like blending, \eg CutMix~\cite{yun2019cutmix} and GridMix~\cite{baek2021gridmix}, (b) methods that generate convex combinations in the learned feature space, \eg \textit{Manifold Mixup} \cite{verma2019manifold} and MixFeat~\cite{yaguchi2019mixfeat}, and (c) methods that add a learnable component to \mixup, \eg AdaMixUp~\cite{guo2019mixup}, AutoMix~\cite{zhu2020automix}, and AutoMix~\cite{liu2021unveiling}. However, \mixup~can lead to ghosting artifacts in the synthesized samples (as we show later in the paper, \eg \fig\ref{fig:mnist_vis}), in addition to generating synthetic samples with wrong class labels. Moreover, because \mixup~uses a convex combination of only a pair of points, it can lead to the synthetic samples being generated off the original data manifold (\fig\ref{fig:overview} (a)). This in turn leads to an inflation of the manifold, which can be quantified by an increase in the intrinsic dimensionality of the resulting data distribution, as shown in \fig\ref{fig:localID}, which is undesirable since it has been shown that deep models trained on datasets with lower dimensionalities generalize better to unseen samples~\cite{pope2021the}. Additionally, \mixup-like approaches, which crop or mask regions of the input images, may degrade the training data quality by occluding informative and discriminatory regions of images, which is highly undesirable for high-stakes applications such as medical image analysis tasks.

The primary hypothesis of \mixup~and many of its derivatives is that a model should behave linearly between any two training samples, even if the distance between samples is large. This implies that we may train the model with synthetic samples that have very low confidence of realism; in effect over-regularizing. We instead argue that a model should only behave linearly nearby training samples and that we should thus only generate synthetic examples with high confidence of realism. To achieve this, we propose \nmixup, a generalization of \mixup~with provably desirable properties that addresses the shortcomings of \mixup. 
\nmixup~generates new training samples by using a convex combination of $N$ samples in a training batch, requires no custom layers or special training procedures to employ, and is faster than \mixup~in terms of wall-clock time.
We show how, as compared to \mixup, the \nmixup~formulation allows for generating more realistic and more diverse samples that better conform to the data manifold (\fig\ref{fig:overview} (b)) with richer labels that incorporate information from multiple classes, and that \mixup~is indeed a special case of \nmixup. We show qualitatively and quantitatively on synthetic and real-world datasets that \nmixup's output better preserves the intrinsic dimensionality of the data than that of \mixup. Finally, we demonstrate the efficacy of \nmixup~on 24 datasets comprising a wide variety of tasks from natural image classification to diagnosis with several medical imaging modalities.

\section{Method}    \label{sec:method}

\textbf{Vicinal Risk Minimization:} Revisiting the concept of risk minimization from Vapnik~\cite{vapnik1999nature}, given $\mathcal{X}$ and $\mathcal{Y}$ as the input data and the target label distributions respectively, and a family of functions $\mathcal{F}$, the supervised learning setting consists of searching for an optimal function $f \in \mathcal{F}: \mathcal{X} \to \mathcal{Y}$, which minimizes the expected value of a given loss function $\mathcal{L}$ over the data distribution $P(x, y); (x, y) \in (\mathcal{X}, \mathcal{Y})$. This expected value of the loss, also known as the expected value of the risk, is given by:
$    R (f) = \int \mathcal{L} \left(f(x), y\right) \ P(x, y) \ dx \ dy.$
In scenarios when the exact distribution $P(x, y)$ is unknown, such as in practical supervised learning settings with a finite training dataset $\{x_i, y_i\}_{i=1}^m$, the common approach is to minimize the risk w.r.t. the empirical data distribution approximated by using delta functions at each sample,
$    R_{\mathrm{emp}} (f) = \frac{1}{m} \sum_{i=1}^m \mathcal{L} \left(f(x_i), y_i\right),$
and this is known as empirical risk minimization (ERM). However, if the data distribution is smooth, as is the case with most real datasets, it is desirable to minimize the risk in the vicinity of the provided samples~\cite{vapnik1999nature,chapelle2001vicinal},
$    R_{\mathrm{vic}} (f) = \frac{1}{\hat{m}} \sum_{i=1}^{\hat{m}} \mathcal{L} \left(f(\hat{x_i}), \hat{y_i}\right)$,
where $\left\{(\hat{x}, \hat{y})\right\}_{i=1}^{\hat{m}}$ are points sampled from the vicinity of the original data distribution, also known as the vicinal distribution $P_{\mathrm{vic}} (x, y)$. This is known as vicinal risk minimization (VRM) and theoretical analysis~\cite{vapnik1999nature,chapelle2001vicinal,zhang2018generalization} has shown that VRM generalizes well when at least one of these two criteria are satisfied: (i) the vicinal data distribution $P_{\mathrm{vic}} (x, y)$ must be a good approximation of the actual data distribution $P (x, y)$, and (ii) the class $\mathcal{F}$ of functions must have a suitably small capacity.
Since modern deep neural networks have up to hundreds of millions of parameters, it is imperative that the former criteria is met.\\
\\
\noindent \textbf{Data Augmentation:} A popular example of VRM is the use of data augmentation for training deep neural networks. For example, applying geometric and intensity-based transformations to images leads to a diverse training dataset allowing the prediction models to generalize well to unseen samples~\cite{shorten2019survey}. However, the assumption of these transformations that points sampled in the vicinity of the original data distribution share the same class label is rather limiting and does not account for complex interactions (\eg proximity relationships) between class-specific data distributions in the input space. Recent approaches based on convex combinations of pairs of samples to synthesize new training samples aim to alleviate this by allowing the model to learn smoother decision boundaries~\cite{verma2019manifold}. Consider the general $\mathcal{K}$-class classification task. \mixup~\cite{zhang2018mixup} synthesizes a new training sample $(\hat{x}, \hat{y})$ from training data samples $(x_i, y_i)$ and $(x_j, y_j)$ as

\begin{equation} \label{eqn:mixup}
    \hat{x} = \lambda x_i + (1 - \lambda) x_j; \ \ \hat{y} = \lambda y_i + (1 - \lambda) y_j
\end{equation}

\noindent where $\lambda \in [0, 1]$. The labels $y_i$, $y_j$ are converted to one-hot encoded vectors to allow for linear interpolation between pairs of labels. However, as we show in our experiments (Sec.~\ref{sec:results}), \mixup~leads to the synthesized points being sampled off the data manifold (\fig\ref{fig:overview} (a)).\\

\noindent\textbf{\nmixup~Formulation:} Going back to the $\mathcal{K}$-class classification task, suppose we are given a set of $N$ points $\{x_i\}_{i=1}^N$ in a $\mathcal{D}$-dimensional ambient space $\mathbb{R^{\mathcal{D}}}$ with the corresponding labels $\{y_i\}_{i=1}^m$ in a label space $\mathcal{L} = \{l_1, \cdots, l_\mathcal{K}\} \in \mathbb{R^{\mathcal{K}}}$. Keeping in line with the manifold hypothesis~\cite{cayton2005algorithms,fefferman2016testing}, which states that complex data manifolds in high dimensional ambient spaces are actually made up of samples from manifolds with low intrinsic dimensionalities, we assume that the $N$ points are samples from $\mathcal{K}$ manifolds $\{\mathcal{M}_i\}_{i=1}^{\mathcal{K}}$ of intrinsic dimensionalities $\{d_i\}_{i=1}^{\mathcal{K}}$, where $d_i << D \ \forall i \in [1, \mathcal{K}]$ (\fig\ref{fig:overview} (a)). We seek an augmentation method that facilitates a denser sampling of each intrinsic manifold $\mathcal{M}_i$, thus generating more real and more diverse samples with richer labels. Following Wood \etal~\cite{wood2021fake,wood2021synthetic}, we consider three criteria for evaluating the quality of synthetic data: 

\noindent \textbf{(i)~realism}: allowing the generation of correctly labeled synthetic samples close to the original samples, ensuring the realism of the synthetic samples,

\noindent\textbf{(ii)~diversity}: 
facilitating the generation of more diverse synthetic samples by allowing exploration of the input space, and 

\noindent\textbf{(iii)}~\textbf{label richness} when generating synthetic samples
while still staying on the manifold of realistic samples.  
Additionally, we aim for: 

\noindent\textbf{(iv)~valid probabilistic labels} from combinations of samples along with

\noindent\textbf{(v)~computationally efficient} (\eg avoiding inter-sample distance calculations) augmentation of training batches.

To this end, we propose to synthesize a new sample $(x_k, y_k)$ as

\begin{equation} \label{eqn:nmixup}
    x_k = \sum_{i=1}^N w_i x_i; \ \ y_k = \sum_{i=1}^N w_i y_i,
\end{equation}

\noindent where $w_i$s are the weights assigned to the $N$ samples. One such weighting scheme that satisfies the aforementioned requirements consists of sample weights from the terms of a $p$-series, \ie $w_i = i^{-p}$, which is a convergent series for $p \ge 1$. Since this implies that the weight assigned to the first sample will be the largest, we want to randomize the order of the samples to ensure that the synthetic samples are not all generated near one original sample. Therefore, building upon the idea of local synthetic instances initially proposed for the augmentation of connectome dataset~\cite{brown2015prediction}, we adopt the following formulation: Given $N$ samples (where $2 \leq N \leq m$ and thus, theoretically, the entire dataset), an $N \times N$ random permutation matrix $\pi$, and the resulting randomized ordering of samples $s = \pi [1, 2, \dots, N]^T$, the weights are defined as

\begin{equation}
    w_i = \frac{s_i^{-\pvar}}{C}, \ \ i \in [1, N]
\end{equation}

\noindent where $C$ is the normalization constant and $\pvar$ is a hyperparameter. As we show in our experiments later, $\pvar$ allows us to control how far the synthetic samples can stray away from the original samples. Moreover, in order to ensure that $y_k$ in \eqn\ref{eqn:nmixup} is a valid probabilistic label, $w_i$ must satisfy $w_i \ge 0 \ \forall i$ and $\sum_{i=1}^N w_i = 1$.
Accordingly, we use $L_1$-normalization and $C = \sum_{j=1}^N j^{-\pvar}$ 
is the $N$-truncated Riemann zeta function~\cite{riemann1859ueber} $\zeta(z)$ evaluated at $z=\pvar$, and call our method \nmixup.
An illustration of \nmixup~for $N=3, \mathcal{D}=3, d_1 = d_2 = d_3 =2$ is shown in \fig\ref{fig:overview}(a). Notice how despite generating convex combinations of samples from disjoint manifolds, the resulting synthetic samples are close to the original ones. A similar observation can be made for $N=4$ and $N=8$ is shown in \fig\ref{fig:overview}(c). Since there exist $N!$ possible $N \times N$ random permutation matrices, given $N$ original samples, \nmixup~can synthesize $N!$ new samples for a single value of $\pvar$, as compared to \mixup~which can only synthesize 1 new sample per sample pair for a single value of $\lambda$. 

As a result of the aforementioned formulation, \nmixup~presents two desirable properties that we present in the following 2 theorems (proofs in the Appendix). Theorem~\ref{theorem:pvalmin} states that for all values of $\pvar \geq \pvarmin$, the weight assigned to one sample is greater than the sum of the weights assigned to all the other samples in a batch, thus implicitly introducing the desired notion of linearity in only the locality of the original samples. Theorem~\ref{theorem:special_case} states the equivalence of \mixup~and \nmixup~and establishes the former as a special case of the latter.

\begin{thm} \label{theorem:pvalmin}
For $\gamma \geq \pvarmin = \pvalmin$, the weight assigned to one sample dominates all other weights, i.e., $\forall \ \gamma \geq \pvalmin$, $w_1 > \sum_{i=2}^N w_i$.
\end{thm}

\begin{thm} \label{theorem:special_case}
For $N = 2$ and $\pvar= \log_2 \left(\frac{\lambda}{1-\lambda}\right)$, 
\nmixup~simplifies to \mixup.
\end{thm}

\begin{figure*}[ht!]
     \centering
     \begin{subfigure}[t]{0.9\textwidth}
         \centering
         \includegraphics[width=\textwidth]{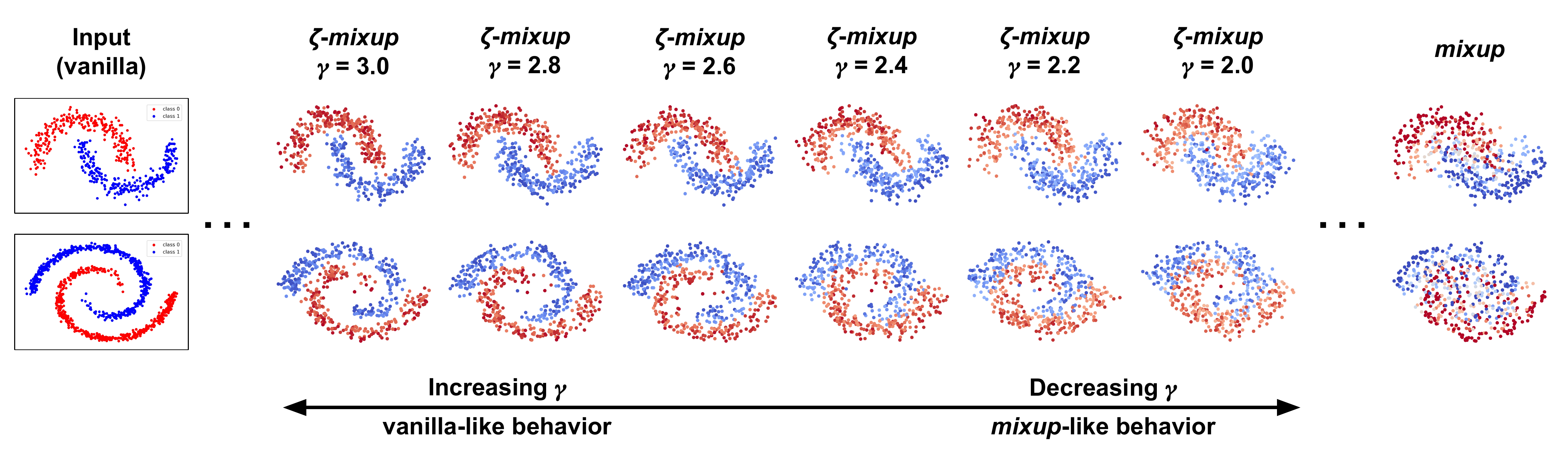}
         \caption{Synthetic two-class 2D data with non-linear class decision boundaries.}
     \end{subfigure}
     \\
      \begin{subfigure}[b]{0.55\textwidth}
         \centering
         \includegraphics[width=\textwidth]{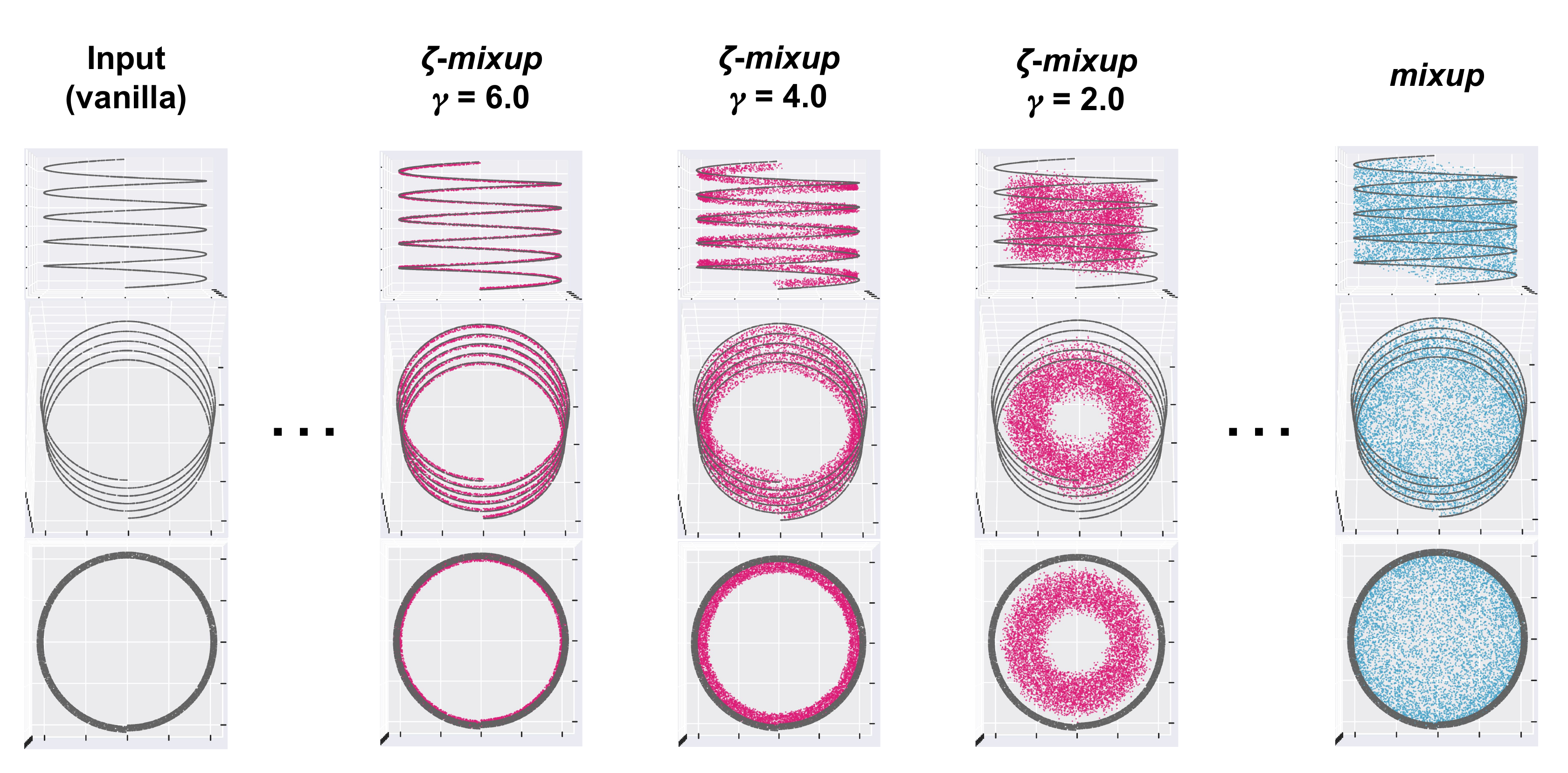}
         \caption{Synthetic data distributed along a 3D helical manifold. 2D projections of the 3D manifolds are shown from the following viewpoints top to bottom: (elevation, azimuth): (0\textdegree, 0\textdegree), (70\textdegree, 0\textdegree), (90\textdegree, 0\textdegree). In all the plots, the grey points denote the original input samples.}
     \end{subfigure}
     \hfill
      \begin{subfigure}[b]{0.4\textwidth}
         \centering
         \includegraphics[width=\textwidth]{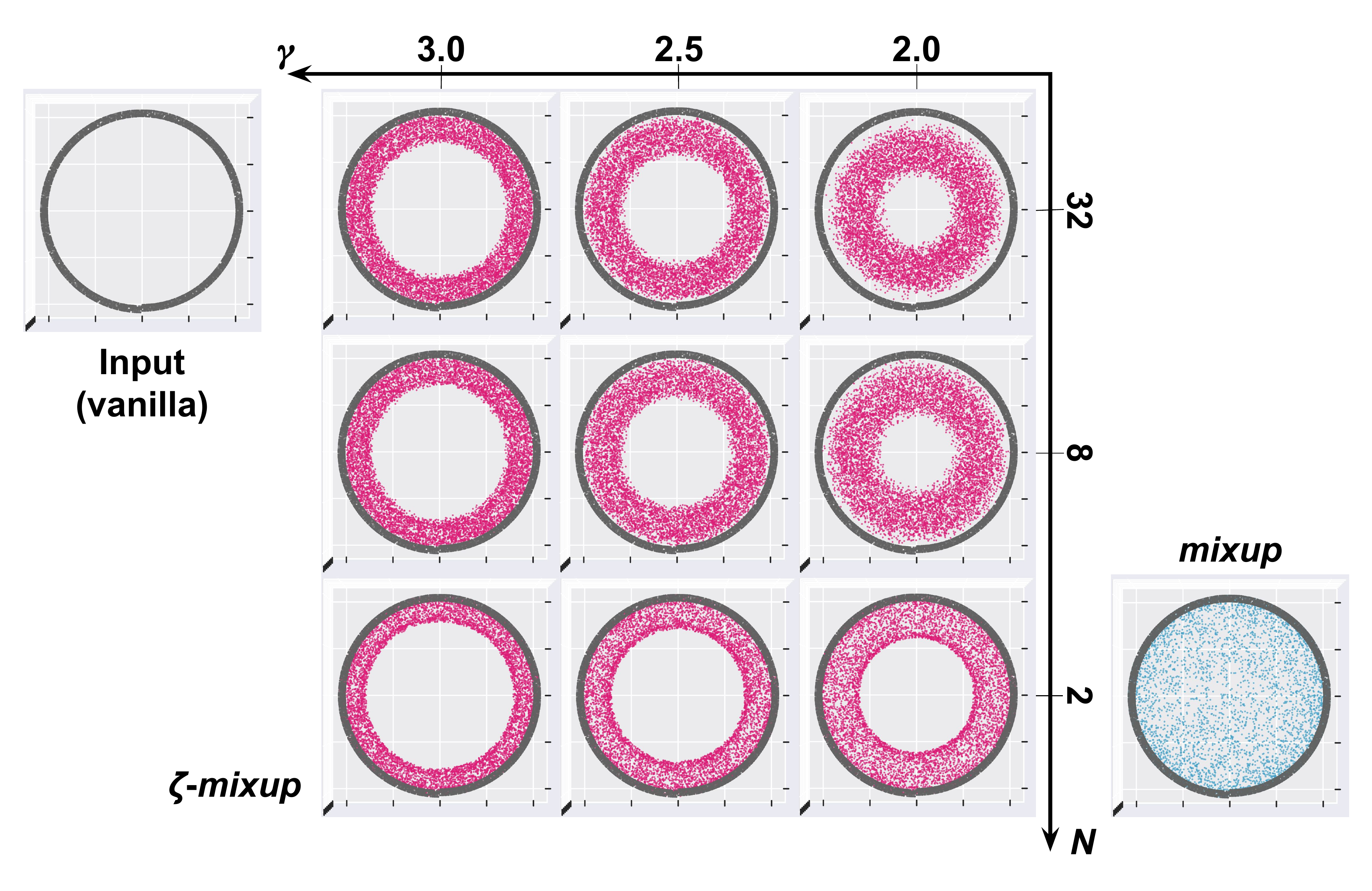}
         \caption{Visualizing the effect of changing $N$ and $\pvar$ on the output of \nmixup. 2D projections of the 3D manifolds are shown from the (elevation, azimuth): (0\textdegree, 90\textdegree) viewpoint.}
     \end{subfigure}
        \caption{Visualizing how \mixup~and \nmixup~synthesize new samples. Notice that \mixup~produces samples that (a) are assigned wrong labels and (b) are sampled off the original data manifold, with an extreme example being where the points are sampled from the hollow region in the helix. A moderately low value of $\pvar$ allows for a more reasonable exploration of the data manifold, with higher values of $N$ allowing more diversity in the synthesized points.
        }
        \label{fig:distribution_vis}
\end{figure*}

\section{Datasets and Experimental Details}

\noindent\textbf{Synthetic Data:} \label{subsec:synth_data}
We first generate two-class distributions of $2^9 = 512$ samples with non-linear class boundaries in the shape of interleaving crescents (CRESCENTS) and spirals (SPIRALS), and add Gaussian noise $\mathcal N(0,0.1)$,
as shown in the ``Input" column of Fig.~\ref{fig:distribution_vis} (a). Next, moving on to higher dimensional spaces, we generate synthetic data distributed along a helix. In particular, we sample $2^{13}$ = 8,192 points off a 1-D helix embedded in $\mathbb{R}^3$ (see the ``Input" column of Fig.~\ref{fig:distribution_vis} (b)) and, as a manifestation of low-D manifolds lying in high-D ambient spaces, a 1-D helix in $\mathbb{R}^{12}$. \\
\\
\noindent\textbf{Natural Image Datasets (NATURAL):}
We use MNIST~\cite{lecun1998gradient}, CIFAR-10 and CIFAR-100~\cite{krizhevsky2009learning}, Fashion-MNIST (F-MNIST)~\cite{xiao2017fashion}, STL-10~\cite{coates2011analysis}, and, to evaluate models on real-world images but with faster training times, two 10-class subsets of the standard ImageNet~\cite{deng2009imagenet}:  Imagenette and Imagewoof~\cite{howard2019imagenette}. 
Further details about these datasets and model training are in the Appendix. We train ResNet-18~\cite{he2016deep} models and report the overall error rate (ERR) since the datasets have balanced class distributions. \\
\\
We use 10 skin lesion image diagnosis datasets: 
ISIC 2016~\cite{gutman2016skin}, ISIC 2017~\cite{codella2018skin}, ISIC 2018~\cite{codella2019skin,tschandl2018ham10000}, 
MSK~\cite{ISICArchive}, 
(all datasets have dermoscopic images, \ie captured by a dermatoscope~\cite{kittler2002diagnostic,menzies2009dermoscopy}, except those denoted by a $\dagger$).
We train ResNet-18 and ResNet-50~\cite{he2016deep} models 
on the 5-class diagnosis task used in the literature~\cite{kawahara2018seven,coppola2020interpreting,abhishek2021predicting} 
and report three evaluation metrics that account for the inherent class imbalance: balanced accuracy (\ie macro-averaged recall)~\cite{mosley2013balanced} (ACC$_{\mathrm{bal}}$) and micro- and macro-averaged F1 scores.\\

\begin{figure*}[ht!]
     \centering
     \begin{subfigure}[t]{0.27\textwidth}
         \centering
         \includegraphics[width=\textwidth]{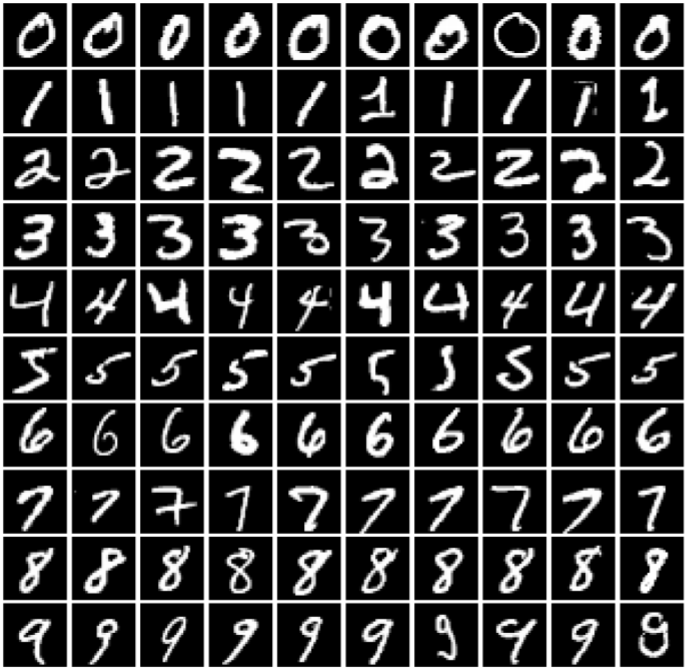}
         \caption{The first 10 images of each MNIST class.}
     \end{subfigure}
     \quad
     \begin{subfigure}[t]{0.335\textwidth}
         \centering
         \includegraphics[width=\textwidth]{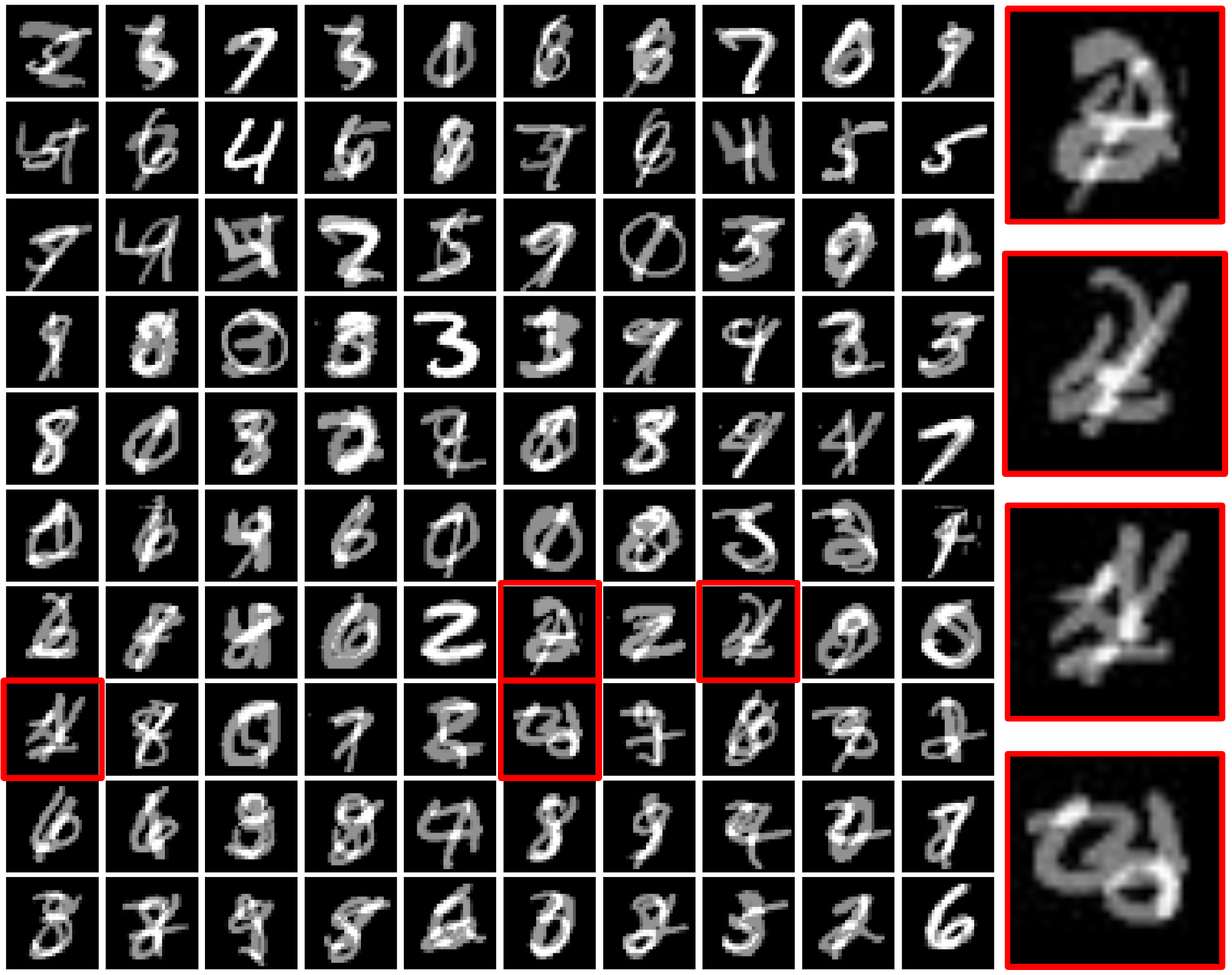}
         \caption{Output of \mixup\xspace($N = 25$).}
     \end{subfigure}
     \quad
     \begin{subfigure}[t]{0.27\textwidth}
         \centering
         \includegraphics[width=\textwidth]{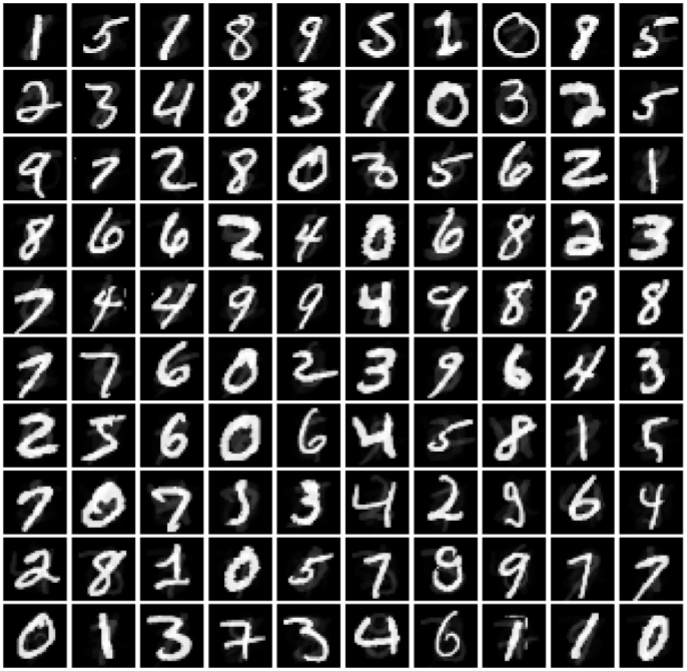}
         \caption{Output of \nmixup\xspace($N = 25$; $\pvar=2.8$).}
     \end{subfigure}
     \\
     \begin{subfigure}[t]{0.24\textwidth}
         \centering
         \includegraphics[width=\textwidth]{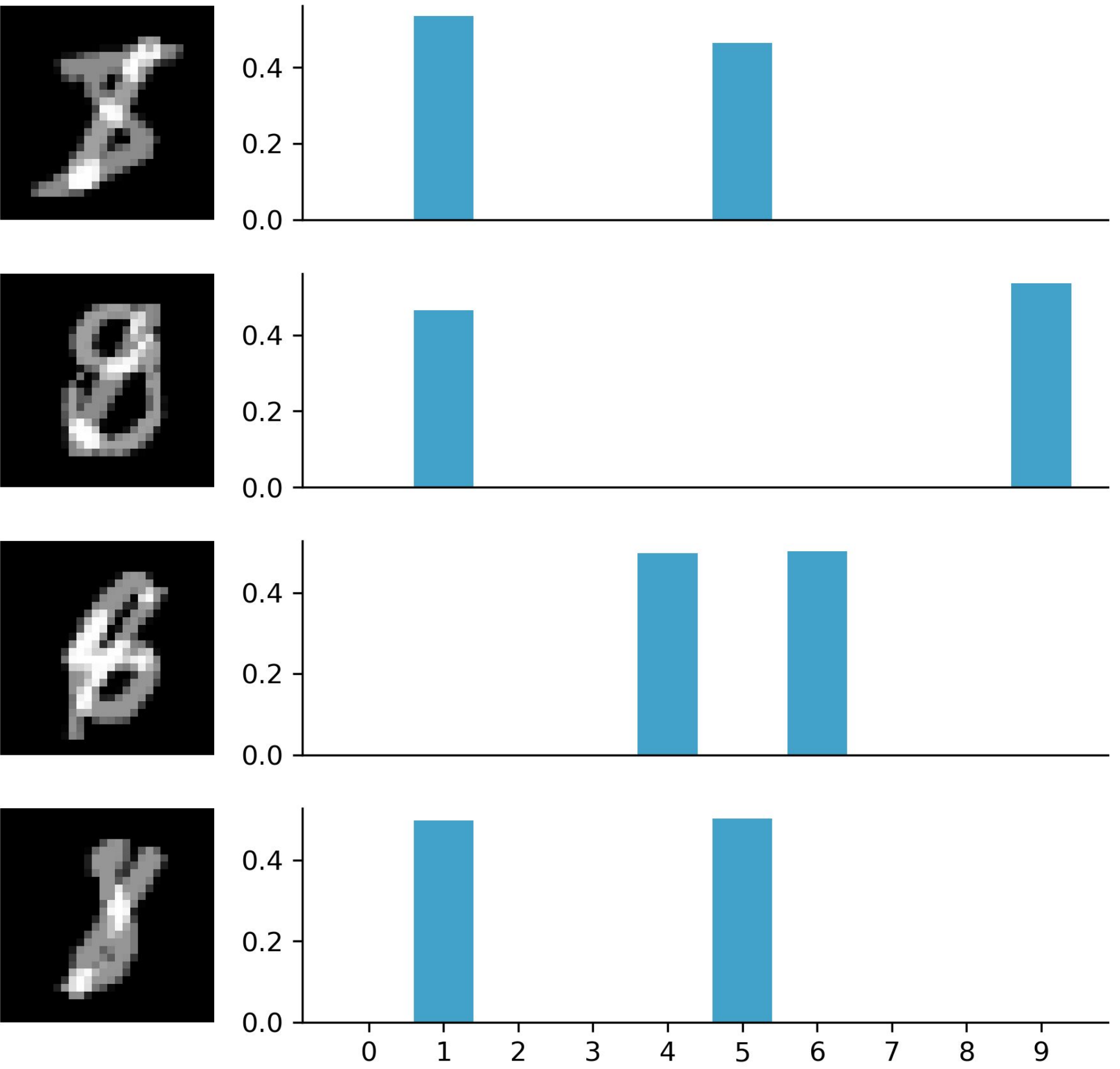}
         \caption{Soft labels of \mixup~outputs.}
     \end{subfigure}
     \hfill
      \begin{subfigure}[t]{0.24\textwidth}
         \centering
         \includegraphics[width=\textwidth]{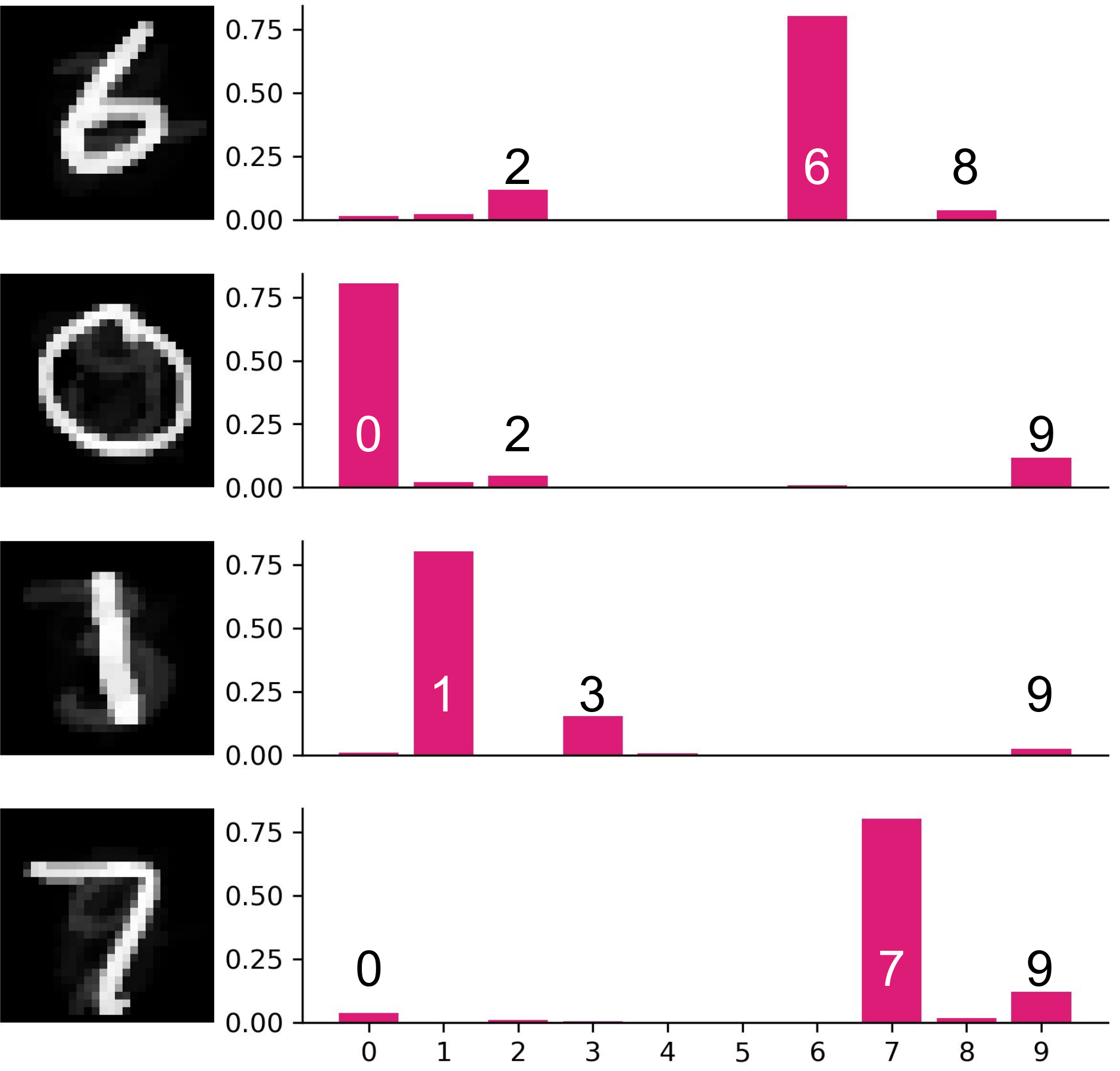}
         \caption{Soft labels of \nmixup~outputs.}
     \end{subfigure}
     \hfill
 \begin{subfigure}[t]{0.24\textwidth}
     \centering
     \includegraphics[width=\textwidth]{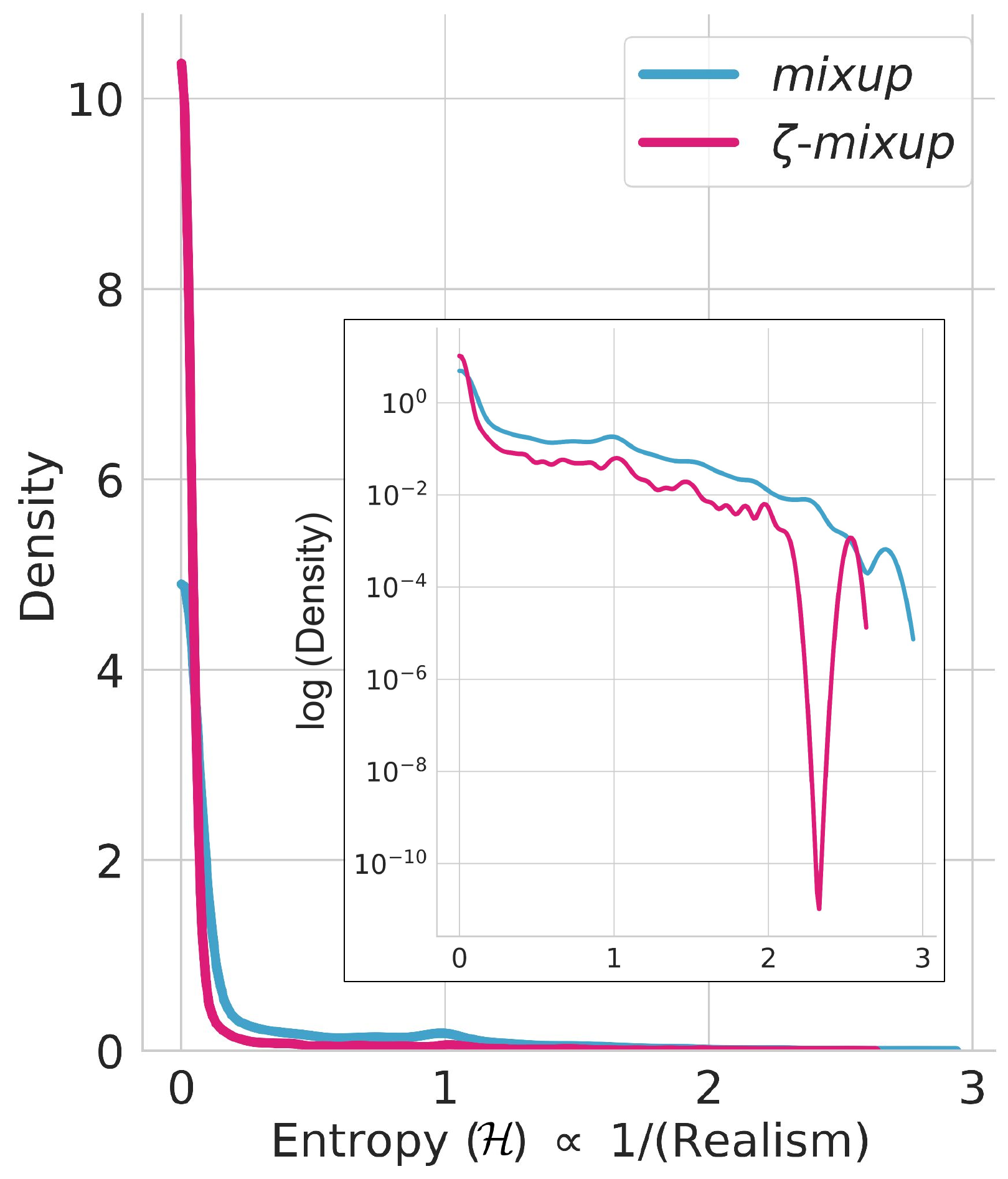}
     \caption{Measuring realism of samples.}
 \end{subfigure}
 \hfill
  \begin{subfigure}[t]{0.25\textwidth}
     \centering
     \includegraphics[width=\textwidth]{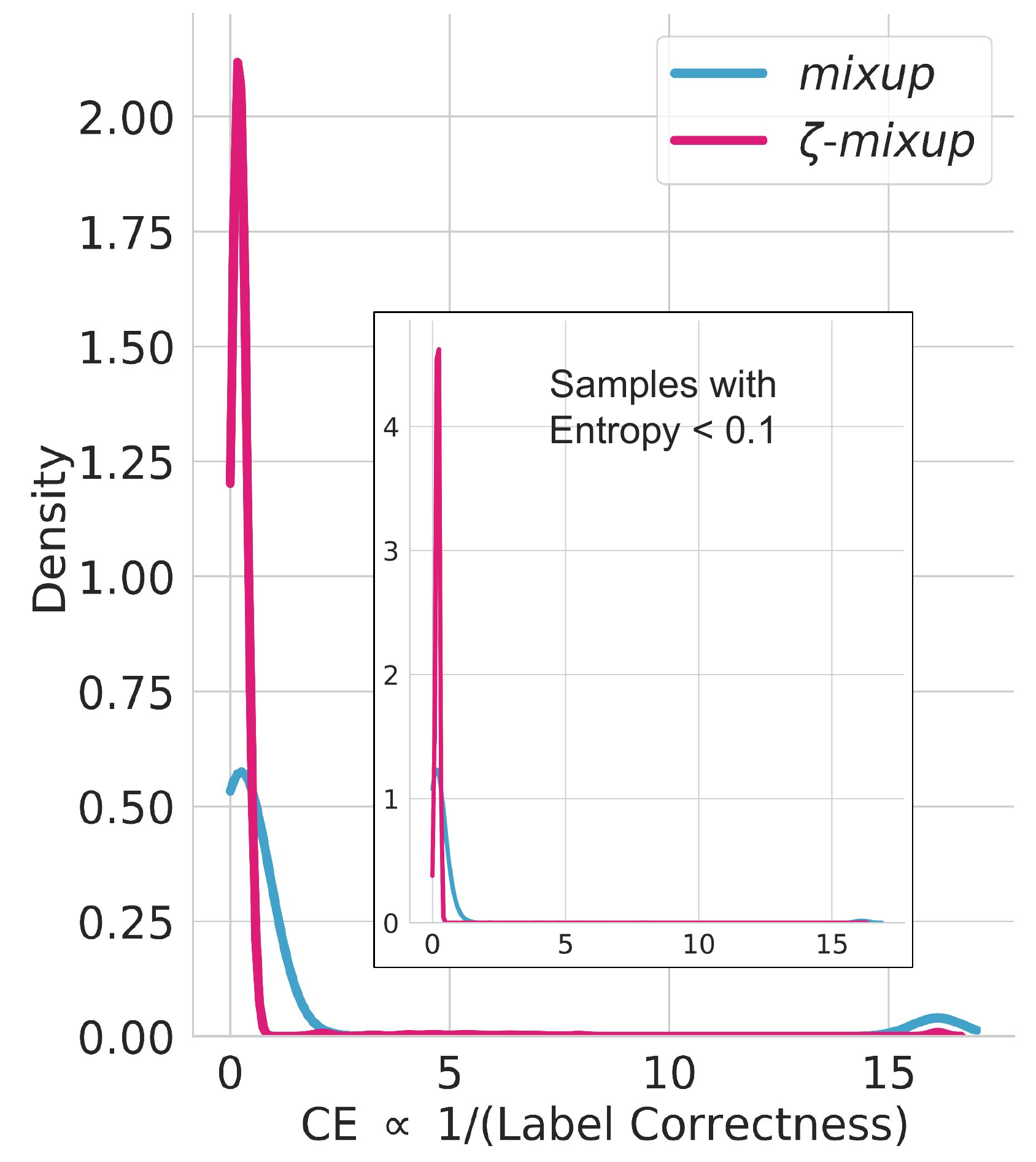}
     \caption{Measuring correctness of labels.}
\end{subfigure}
        \caption{Visualizing the results obtained using \mixup~and \nmixup~on images from the MNIST dataset. In (d) and (e), we visualize the probabilistic ``soft" labels assigned to images generated by \mixup~and \nmixup~respectively. Notice how all images in (d) look close to the digit ``8" while their assigned soft labels do not contain the class ``8". \nmixup~alleviates this issue and the soft labels in (e) correspond exactly to the class the synthesized images belong to.
        Also note how \mixup~produces images with a wrong label, i.e., a label different from the original labels of the 2 images it is interpolated from. In (f) and (g), we evaluate the realism of \mixup's and \nmixup's generated samples and the correctness of the corresponding labels by measuring the entropy of the Oracle's predictions ($\mathcal{H}$) and the cross entropy of the Oracle's predictions with the soft labels (CE) respectively. For both (f) and (g), lower values are better.}
        \label{fig:mnist_vis}
\end{figure*}

\noindent\textbf{Datasets of Other Medical Imaging Modalities (MEDMNIST):}
To evaluate our models on multiple medical imaging modalities, we use the 8 datasets from the MedMNIST Classification Decathlon~\cite{yang2021medmnist}: PathMNIST$^\ddagger$ (histopathology images), DermaMNIST$^\ddagger$ (multi-source images of pigmented skin lesions), OCTMNIST (optical coherence tomography images), PneumoniaMNIST (pediatric chest X-ray images), BreastMNIST (breast ultrasound images), and OrganMNIST\_\{A, C, S\} (axial, coronal, and sagittal views respectively of 3D computed tomography CT scans). Datasets denoted by $\ddagger$ consist of RGB images, others are grayscale. We train ResNet-18~\cite{he2016deep} models and report overall accuracy (ACC) and area under the ROC curve (AUC), similar to the MedMNIST paper~\cite{yang2021medmnist}.

\begin{figure*}[ht!]
     \centering
     \begin{subfigure}[t]{0.23\textwidth}
         \centering
         \includegraphics[width=\textwidth]{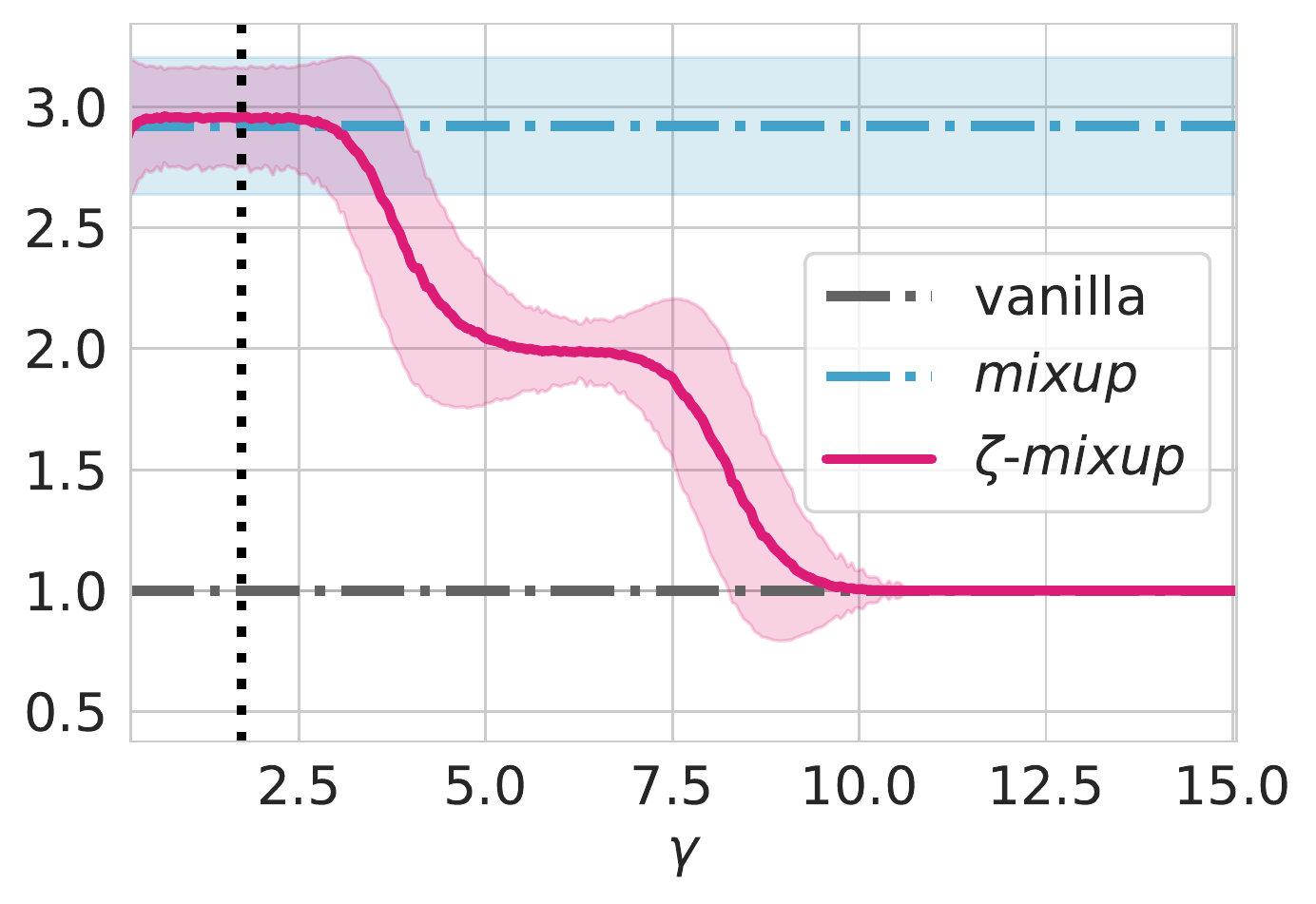}
         \caption{A 1D helix in a 3D embedding space ($n_{\mathrm{NN}} = 8$).}
     \end{subfigure}
     \hfill
     \begin{subfigure}[t]{0.23\textwidth}
         \centering
         \includegraphics[width=\textwidth]{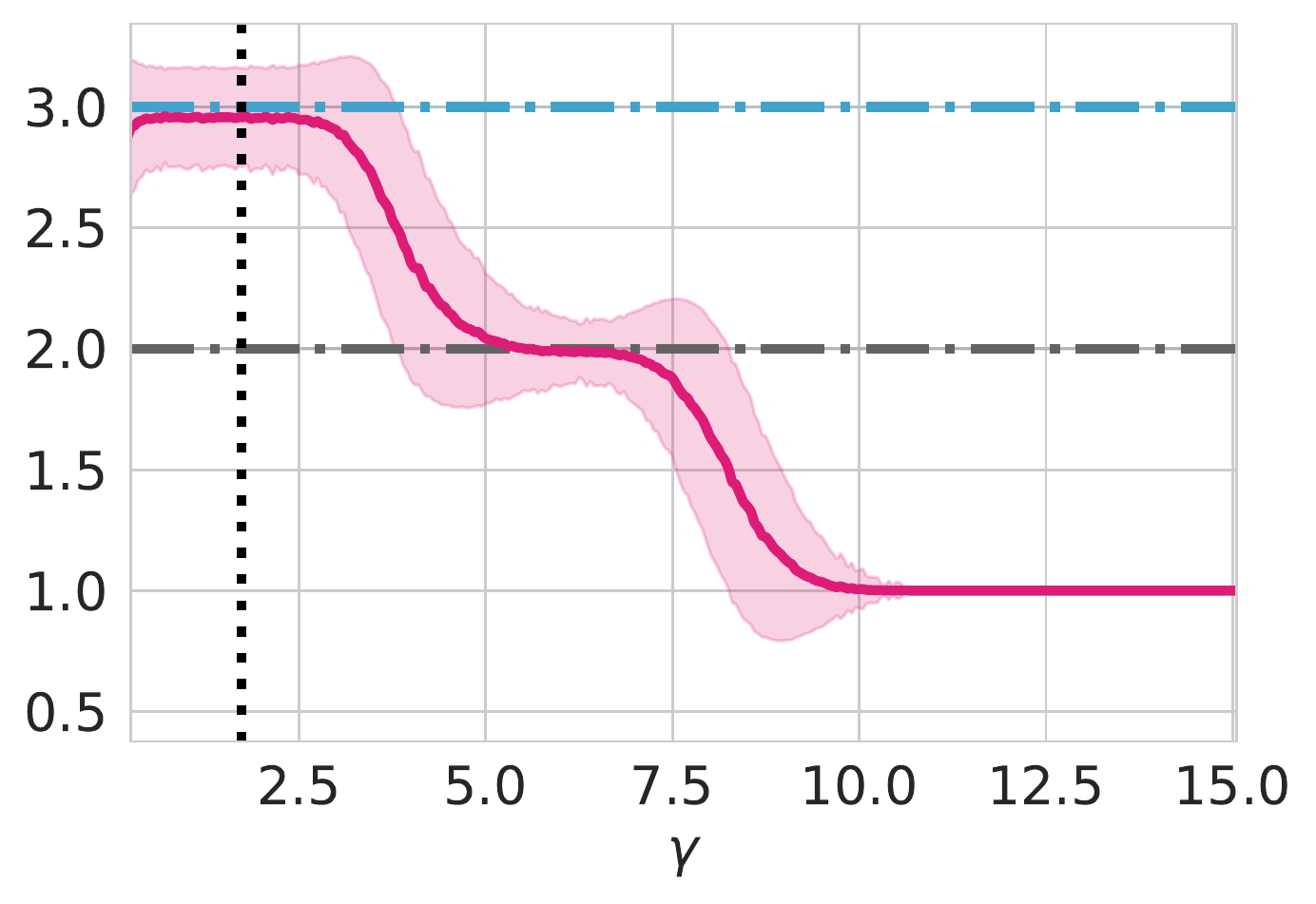}
         \caption{A 1D helix in a 3D embedding space ($n_{\mathrm{NN}} = 128$).}
     \end{subfigure}
     \hfill
     \begin{subfigure}[t]{0.23\textwidth}
         \centering
         \includegraphics[width=\textwidth]{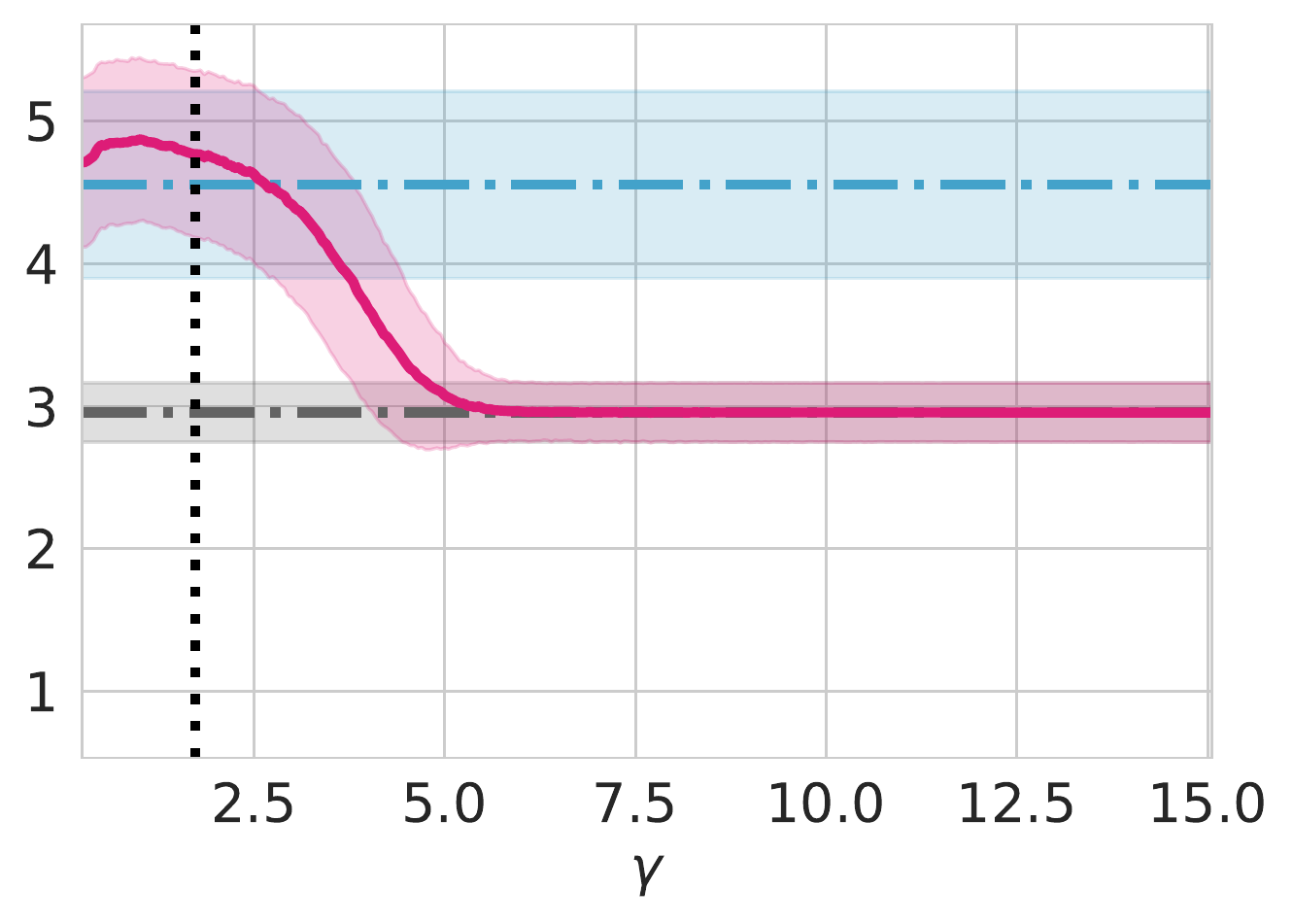}
         \caption{A 3D highly curved manifold in a 12D embedding space ($n_{\mathrm{NN}} = 8$).}
     \end{subfigure}
     \hfill
     \begin{subfigure}[t]{0.23\textwidth}
         \centering
         \includegraphics[width=\textwidth]{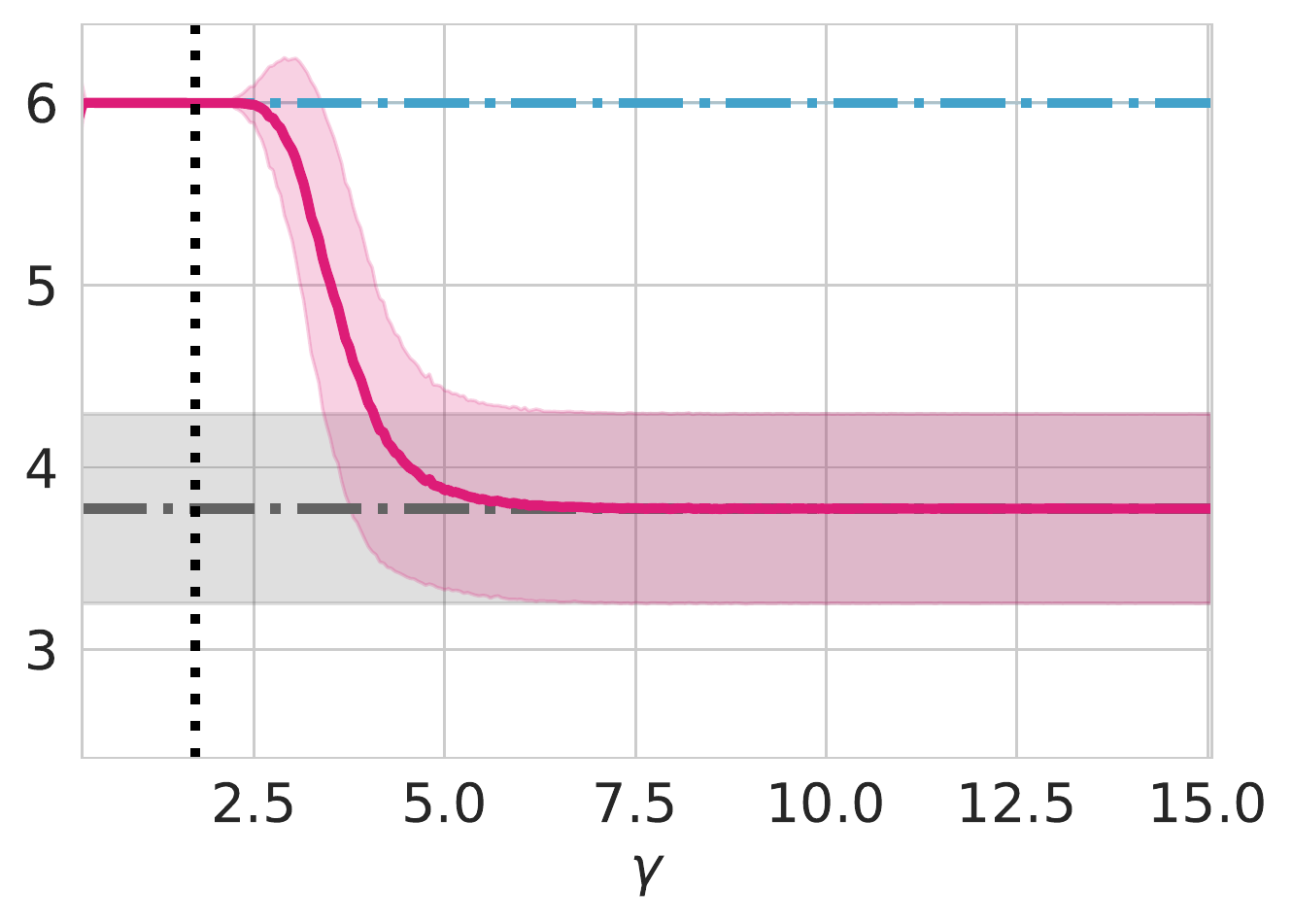}
         \caption{A 3D highly curved manifold in a 12D embedding space ($n_{\mathrm{NN}} = 128$).}
     \end{subfigure}
     \\
      \begin{subfigure}[t]{0.23\textwidth}
         \centering
         \includegraphics[width=\textwidth]{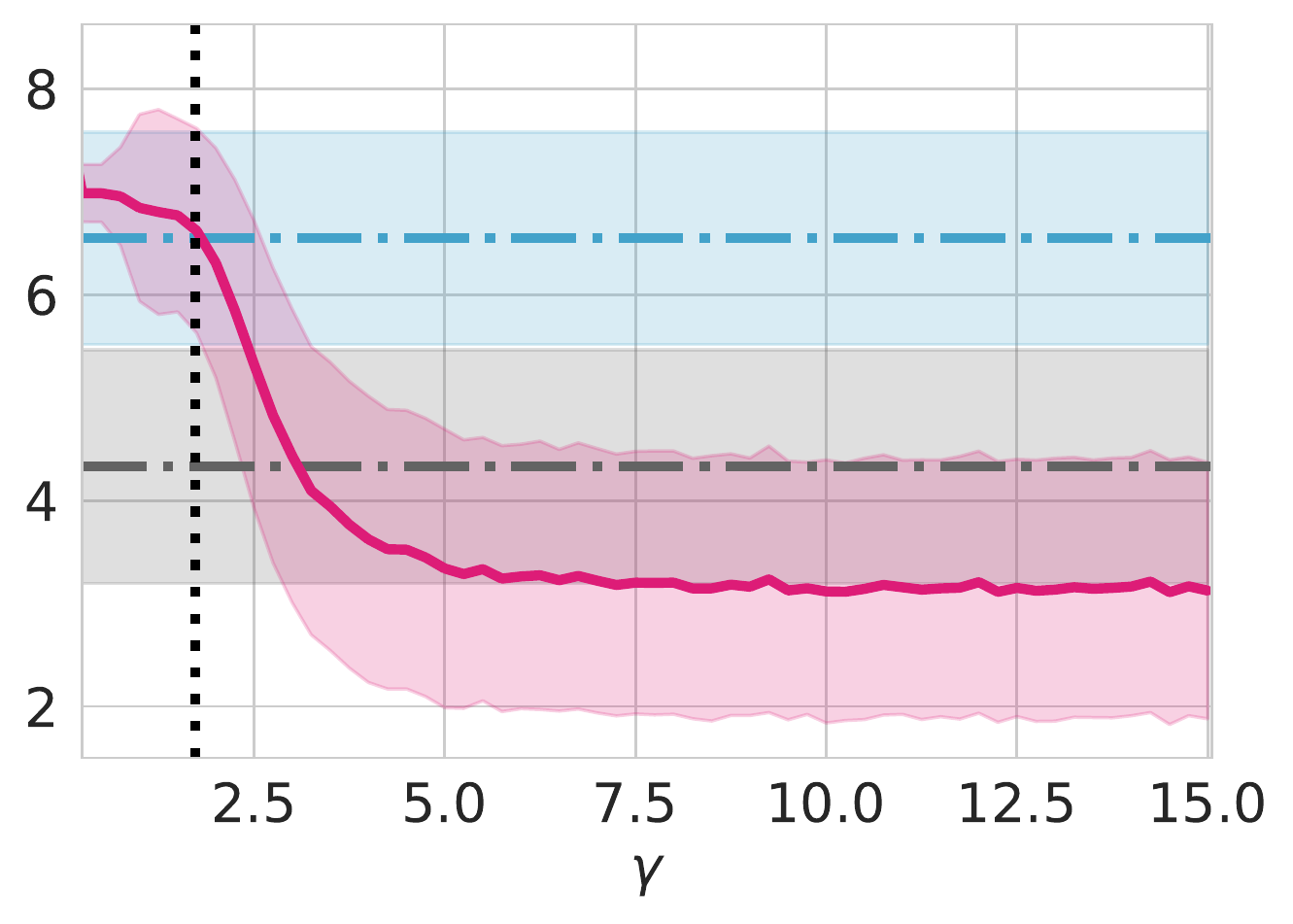}
         \caption{CIFAR-10 images ($n_{\mathrm{NN}} = 8$).}
     \end{subfigure}
     \hfill
     \begin{subfigure}[t]{0.23\textwidth}
         \centering
         \includegraphics[width=\textwidth]{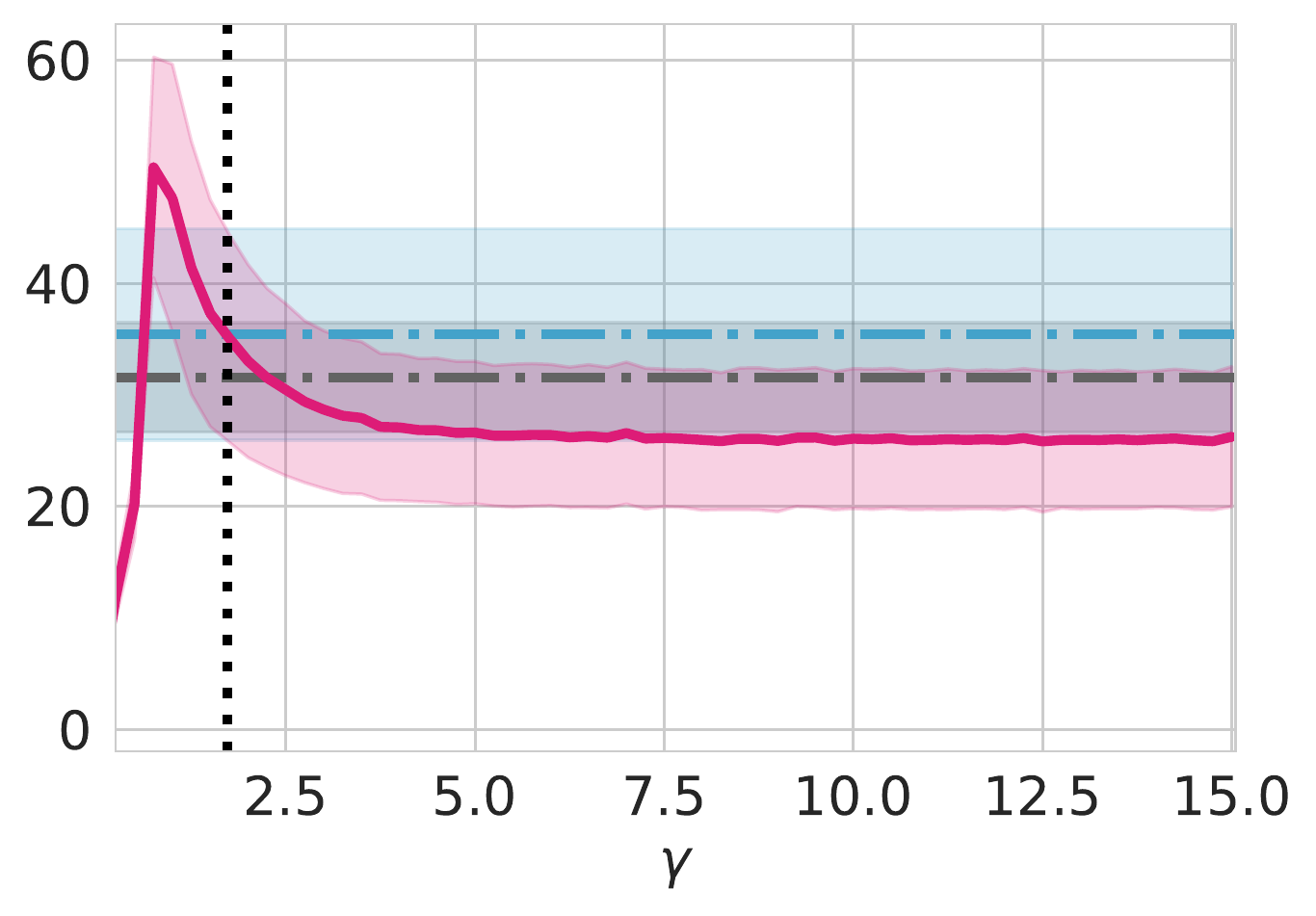}
         \caption{CIFAR-10 images ($n_{\mathrm{NN}} = 128$).}
     \end{subfigure}
     \hfill
     \begin{subfigure}[t]{0.23\textwidth}
         \centering
         \includegraphics[width=\textwidth]{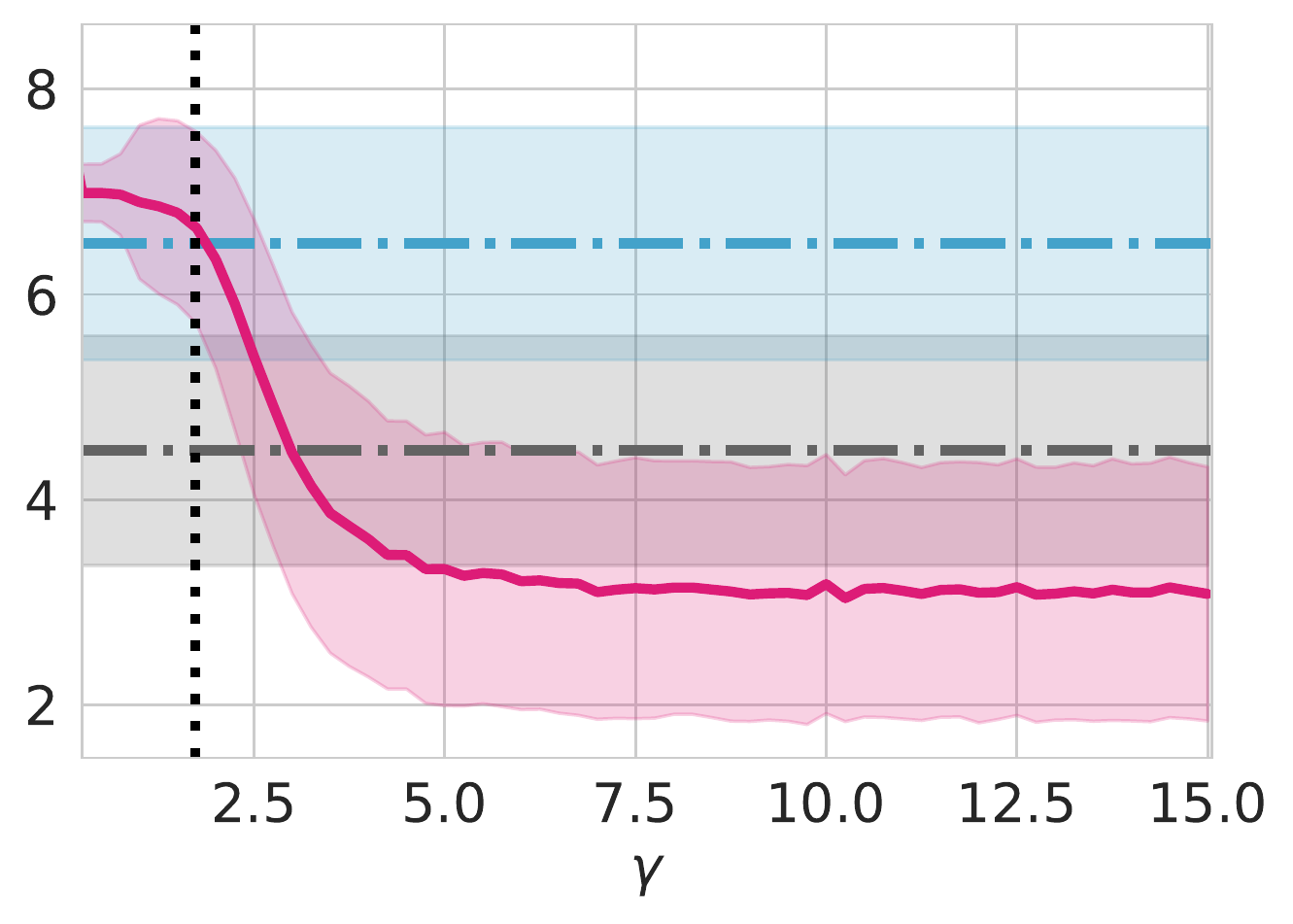}
         \caption{CIFAR-100 images ($n_{\mathrm{NN}} = 8$).}
     \end{subfigure}
     \hfill
     \begin{subfigure}[t]{0.23\textwidth}
         \centering
         \includegraphics[width=\textwidth]{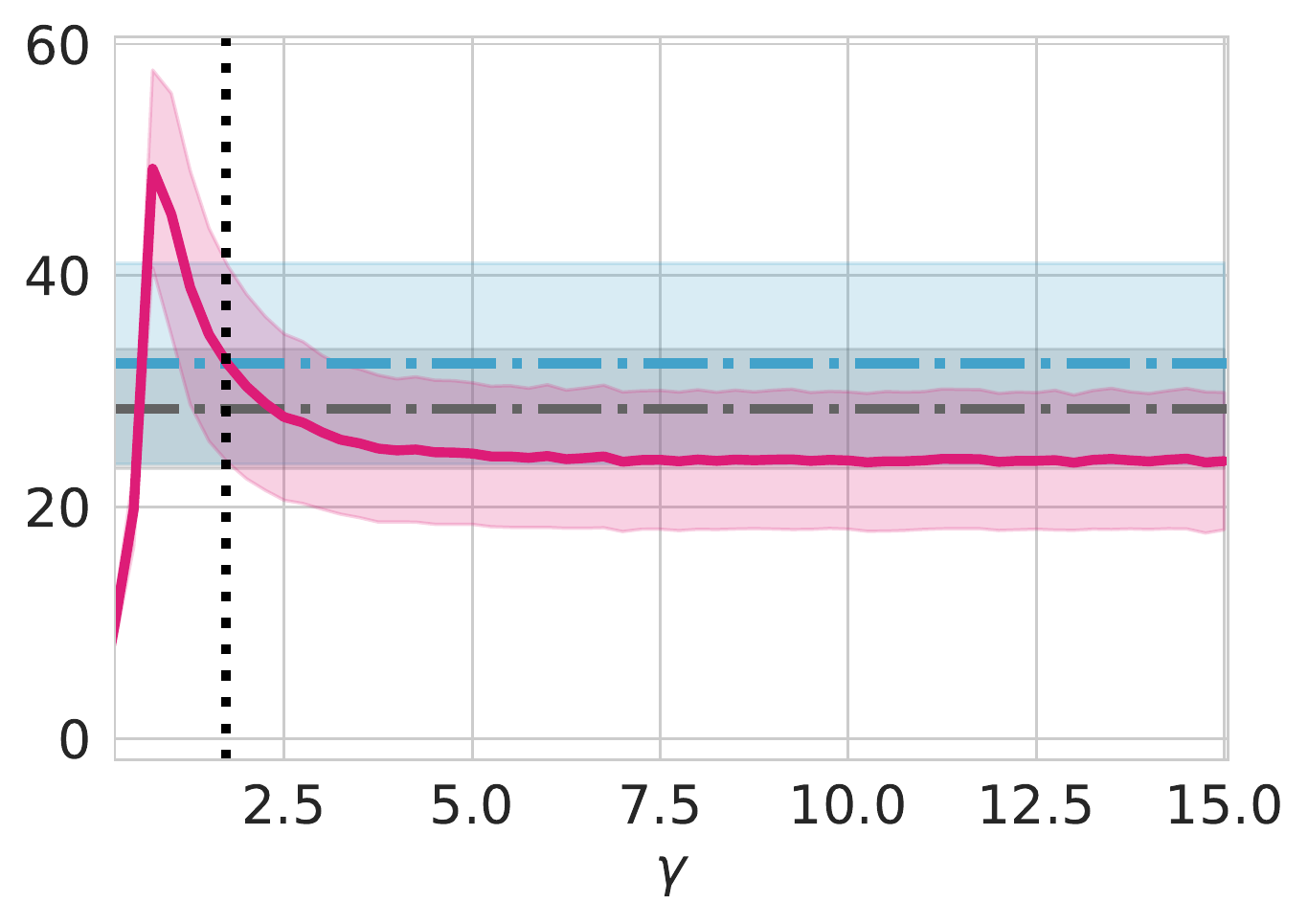}
         \caption{CIFAR-100 images ($n_{\mathrm{NN}} = 128$).}
     \end{subfigure}
        \caption{Visualizing how \nmixup~affects the local intrinsic dimensionality of synthesized datasets distributed as 1D helices ((a) and (b)) and 3D manifold ((c)) in a higher dimensional embedding space as the hyperparameter $\pvar$ changes. The mean and the standard deviation of the intrinsic dimensionality are shown using lines (bold or dashed-dotted) and shaded bands respectively. The vertical dotted line in all the plots denotes the value of $\pvar=\pvarmin$ (Theorem~\ref{theorem:pvalmin}).}
        \label{fig:localID}
\end{figure*}

\section{Results and Discussion}    \label{sec:results}

We present experimental evaluation on controlled synthetic (1-D manifolds in 2-D and 3-D, 3-D manifolds in 12-D) and on 24 real-world natural and medical image datasets of various modalities. We evaluate the quality of \nmixup's outputs: directly, by assessing the realism, label correctness, diversity, richness~\cite{wood2021synthetic,wood2021fake}, and preservation of intrinsic dimensionality of the generated samples; as well as indirectly, by assessing the effect of the samples on the performance of downstream classification tasks. 
    
\begin{table*}[ht]
\centering
\caption{Classification error rates (ERR) on NATURAL. The lowest and the second lowest errors are formatted with \textbf{bold} and {\ul underline} respectively. Improvements over \mixup~are shown in green. ERRs are averages over 3 runs.}
\label{tab:natural_resultsV2}
\resizebox{0.9\textwidth}{!}{%
\setlength{\tabcolsep}{0.5em}
\def\arraystretch{1.5}
\begin{tabular}{cllllll}

\toprule
\textbf{Method}                          & \multicolumn{1}{l}{CIFAR-10}         & CIFAR-100  & F-MNIST  & STL-10  & Imagenette  & Imagewoof  \\ 
\# images (\#classes) & 60,000 (10) & 60,000 (10) & 60,000 (10) & 13,000 (10) & 13,394 (10) & 12,954 (10) \\
\midrule
ERM            & \multicolumn{1}{l}{$5.48$} & $23.33$  & $6.11$  & $25.74$  & $16.08$  & $30.92$ \\ 
\mixup          & \multicolumn{1}{l}{$4.68$} & $21.85$  & $6.04$  & $25.31$  & $16.20$  & $30.80$ \\ 
\nmixup ($\pvar=2.4$)        & \multicolumn{1}{l}{\bm{$4.42\Rise{5.56\%}$}} & {$21.50\Rise{1.60\%}$}  & {$6.04\rise{0.00\%}$}  & \bm{$24.14\Rise{4.62\%}$}  & \bm{$15.16\Rise{6.42\%}$}  & {$30.72\Rise{0.26\%}$} \\ 
\nmixup ($\pvar=2.8$)        & \multicolumn{1}{l}{\ul $4.67$}$\Rise{0.21\%}$ & {\ul $21.35$}$\Rise{2.29\%}$  & \bm{$5.70\Rise{5.63\%}$}  & {\ul $24.82$}$\Rise{1.94\%}$  & {\ul $15.62$}$\Rise{3.58\%}$  & \bm{$30.21\Rise{1.92\%}$} \\ 
\nmixup ($\pvar=4.0$)       & \multicolumn{1}{l}{\bm{$4.42\Rise{5.56\%}$}} & \bm{$21.28\Rise{2.61\%}$}  & {\ul $5.89$}$\Rise{2.48\%}$  & {$24.92\Rise{1.54\%}$}  & {$15.92\Rise{1.73\%}$}  & {\ul $30.67$}$\Rise{0.42\%}$ \\ 
\bottomrule
\end{tabular}
}
\end{table*}

\subsection{Realism and Label Correctness}  \label{subsec:realism}
    While it is desirable that the output of any augmentation method be different from the original data in order to better minimize $R_{\mathrm{vic}}$ (Sec.~\ref{sec:method}), we want to avoid sampling synthetic points off the original data manifold. 
    Applying \mixup~to CRESCENTS and SPIRALS datasets shows that \mixup~does not respect the individual class boundaries and synthesizes samples off the data manifold, also known as manifold intrusion~\cite{guo2019mixup}. This also results in the generated samples being wrongly labeled, \ie points in the ``red" class's region being assigned ``blue" labels and vice versa, which we term as ``label error". On the other hand, \nmixup~preserves the class decision boundaries irrespective of the hyperparameter $\pvar$ and additionally allows for a controlled interpolation between the original distribution and \mixup-like output.
    With \nmixup, small values of $\pvar$ (greater than $\pvarmin$; see Theorem~\ref{theorem:pvalmin}) lead to samples being generated further away from the original data and as $\pvar$ increases, the resulting distribution approaches the original data.
    
    Applying \mixup~in 3D space (\fig\ref{fig:distribution_vis} (b)) results in a somewhat extreme case of the generated points sampled off the data manifold, filling up the entire hollow region in between the helical distribution. \nmixup, however, similar to Fig.~\ref{fig:distribution_vis} (a), generates points that are relatively much closer to the original points, and increasing the value of $\pvar$ to a large value, say $\pvar=6.0$, leads the generated samples to lie almost perfectly on the original data manifold.
    
    Moving on to higher dimensions with the MNIST data, \ie 784-D, we observe that the problems with \mixup's output are even more severe and that the improvements by using \nmixup~are more conspicuous. For each digit class in the MNIST dataset, we take the first 10 samples as shown in \fig\ref{fig:mnist_vis} (a) and use \mixup~and \nmixup~to generate 100 new images each (\fig\ref{fig:mnist_vis} (b-c)). It is easy to see that the digits in \nmixup's output are more discernible than those in \mixup's output.
    
    Finally, to analyze the correctness of probabilistic labels in the outputs of \mixup~and \nmixup, we pick 4 samples from each.
    \mixup's outputs (\fig\ref{fig:mnist_vis} (d)) all look like images of handwritten ``8". The soft label of the first digit in \fig\ref{fig:mnist_vis} (d) is $[0, 0.53, 0, 0, 0, 0.47, 0, 0, 0, 0]$, where the $i^{\mathrm{th}}$ index is the probability of the $i^{\mathrm{th}}$ digit, implying that this output has been obtained by mixing images of digits ``1" and ``5". Interestingly, neither the resulting output looks like the digits ``1" or ``5" nor is the digit ``8" one of the classes used as input for this image. I.e., there is a disagreement, with \mixup, between the appearance of the synthesized image and its assigned label. Similar label error exists in the other images in \fig\ref{fig:mnist_vis} (d). On the other hand, there is a clear agreement between the images produced by \nmixup~and the labels assigned to them (\fig\ref{fig:mnist_vis} (e)).

\begin{table*}[ht]
\centering
\caption{Classification performance (ACC$_{\mathrm{bal}}$) evaluated on SKIN.
The highest and the second highest values of ACC$_{\mathrm{bal}}$ have been formatted with \textbf{bold} and {\ul underline} respectively. 
}
\label{tab:skin_results_narrow}
\resizebox{0.85\textwidth}{!}{%
\setlength{\tabcolsep}{0.5em}
\begin{tabular}{ccccccc}
\toprule
\multicolumn{2}{c}{\textbf{Dataset}}                                               & \multicolumn{1}{c}{ISIC 2016} & \multicolumn{1}{c}{ISIC 2017}                                                     & \multicolumn{1}{c}{ISIC 2018}                                                        & \multicolumn{1}{c}{MSK}      & \multicolumn{1}{c}{UDA}      \\ 
\multicolumn{2}{c}{\#images (\#classes)}                                        & \multicolumn{1}{c}{1,279 (2)}  & \multicolumn{1}{c}{2,750 (3)}                                                      & \multicolumn{1}{c}{10,015 (5)}                                                        & \multicolumn{1}{c}{3,551 (4)} & \multicolumn{1}{c}{601 (2)}  \\ \hline
\parbox[t]{2mm}{\multirow{5}{*}{\rotatebox[origin=c]{90}{ResNet-18}}}                       & ERM                      & 70.44\%                        & 69.31\%                                                                            & 84.31\%                                                                               & 62.35\%                       & 67.46\%                       \\
                                                 & \mixup~                   & 71.77\%                        & 71.60\%                                                                            & 83.96\%                                                                               & 63.59\%                       & 69.38\%                       \\
                                                 & \nmixup~(2.4)            & \textbf{74.53\%}               & \textbf{73.02\%}                                                                   & \textbf{87.20\%}                                                                      & \textbf{65.52\%}              & \textbf{70.54\%}              \\
                                                 & \nmixup~(2.8)            & {\ul 73.03\%}                  & {\ul 72.33\%}                                                                      & {\ul 84.67\%}                                                                         & {\ul 64.87\%}                 & {\ul 70.22\%}                 \\
                                                 & \nmixup~(4.0)            & 72.27\%                        & 70.93\%                                                                            & 83.63\%                                                                               & 62.39\%                       & 67.88\%                       \\ \hdashline 
\multicolumn{1}{c}{\parbox[t]{2mm}{\multirow{5}{*}{\rotatebox[origin=c]{90}{ResNet-50}}}} & \multicolumn{1}{c}{ERM} & \multicolumn{1}{c}{71.75\%}   & \multicolumn{1}{c}{68.20\%}                                                       & \multicolumn{1}{c}{81.28\%}                                                          & \multicolumn{1}{c}{63.86\%}  & \multicolumn{1}{c}{66.85\%}  \\ 
\multicolumn{1}{c}{}                           & \mixup~                   & 72.08\%                        & 71.51\%                                                                            & 85.65\%                                                                               & {\ul 65.62\%}                 & 67.27\%                       \\
\multicolumn{1}{c}{}                           & \nmixup~(2.4)            & 71.52\%                        & \textbf{72.91\%}                                                                   & 84.75\%                                                                               & 65.23\%                       & {\ul 68.39\%}                 \\
\multicolumn{1}{c}{}                           & \nmixup~(2.8)            & \textbf{72.20\%}               & 69.99\%                                                                            & {\ul 86.59\%}                                                                         & \textbf{65.94\%}              & \textbf{70.92\%}              \\
\multicolumn{1}{c}{}                           & \nmixup~(4.0)            & {\ul 72.11\%}                  & {\ul 72.39\%}                                                                      & \textbf{89.18\%}                                                                      & 65.33\%                       & 67.59\%                       \\ 
\midrule
\multicolumn{2}{c}{\textbf{Dataset}}                                               & \multicolumn{1}{c}{DermoFit}  & \multicolumn{1}{c}{\def\arraystretch{1.0} \begin{tabular}[c]{@{}c@{}}derm7point-C\end{tabular}} & \multicolumn{1}{c}{\def\arraystretch{1.0}\begin{tabular}[c]{@{}c@{}}derm7point-D\end{tabular}} & \multicolumn{1}{c}{PH2}      & \multicolumn{1}{c}{MED-NODE} \\ 
\multicolumn{2}{c}{\#images (\#classes)}                                        & \multicolumn{1}{c}{1,300 (5)}  & \multicolumn{1}{c}{1,011 (5)}                                                      & \multicolumn{1}{c}{1,011 (5)}                                                         & \multicolumn{1}{c}{200 (2)}  & \multicolumn{1}{c}{170 (2)}  \\ \hline
\parbox[t]{2mm}{\multirow{5}{*}{\rotatebox[origin=c]{90}{ResNet-18}}}                       & ERM                      & 80.43\%                        & 42.08\%                                                                            & 54.79\%                                                                               & 84.38\%                       & 75.00\%                       \\
                                                 & \mixup~                   & 81.17\%                        & 46.68\%                                                                            & 55.38\%                                                                               & {\ul 85.94\%}                 & 80.36\%                       \\
                                                 & \nmixup~(2.4)            & 82.57\%                        & {\ul 47.82\%}                                                                      & {\ul 55.88\%}                                                                         & {\ul 85.94\%}                 & 79.29\%              \\
                                                 & \nmixup~(2.8)            & {\ul 83.50\%}                  & \textbf{48.91\%}                                                                   & \textbf{56.41\%}                                                                      & \textbf{96.88\%}              & \textbf{82.86\%}                 \\
                                                 & \nmixup~(4.0)            & \textbf{83.94\%}               & 46.93\%                                                                            & 55.45\%                                                                               & {\ul 85.94\%}                 & {\ul 81.79\%}                       \\ \hdashline 
\multicolumn{1}{c}{\parbox[t]{2mm}{\multirow{5}{*}{\rotatebox[origin=c]{90}{ResNet-50}}}} & \multicolumn{1}{c}{ERM} & \multicolumn{1}{c}{83.24\%}   & \multicolumn{1}{c}{42.15\%}                                                       & \multicolumn{1}{c}{74.64\%}                                                          & \multicolumn{1}{c}{84.38\%}  & \multicolumn{1}{c}{55.46\%}  \\ 
\multicolumn{1}{c}{}                           & \mixup~                   & 84.37\%                        & 45.57\%                                                                            & 62.08\%                                                                               & {\ul 85.94\%}                 & \textbf{81.79\%}                 \\
\multicolumn{1}{c}{}                           & \nmixup~(2.4)            & {\ul 86.26\%}                  & {\ul 46.63\%}                                                                      & \textbf{64.59\%}                                                                      & \textbf{87.50\%}              & {\ul 80.71\%}                       \\
\multicolumn{1}{c}{}                           & \nmixup~(2.8)            & 85.91\%                        & \textbf{48.36\%}                                                                   & {\ul 62.98\%}                                                                         & \textbf{87.50\%}              & \textbf{81.79\%}              \\
\multicolumn{1}{c}{}                           & \nmixup~(4.0)            & \textbf{88.16\%}               & 45.95\%                                                                            & 62.58\%                                                                               & \textbf{87.50\%}              & {\ul 80.71\%}                       \\ 
\bottomrule
\end{tabular}%
}
\end{table*}
    
    Next, we set out to quantify \textbf{(i) realism} and \textbf{(ii) label correctness} of \mixup~and \nmixup-synthesized images.
    To this end, we assume access to an Oracle that can recognize MNIST digits. For \textbf{(i)}, we hypothesize that the more an image is realistic, the more the Oracle will be certain about the digit in it, and vice-versa. 

    For example, although the first image in \fig\ref{fig:mnist_vis} (d) is a combination of a ``1" and a ``5", the resulting image looks very similar to a realistic handwritten ``8". On the other hand, consider the highlighted and zoomed digits in \fig\ref{fig:mnist_vis} (b). For an Oracle, images like these are ambiguous and do not belong to one particular class. Consequently, the uncertainty of the Oracle's prediction will be high.
    We therefore adopt the Oracle's entropy ($\mathcal{H}$)
    as a proxy for realism. 
    For \textbf{(ii)}, we use cross entropy (CE) to compare the soft labels assigned by either \mixup~or \nmixup~to the label assigned by the Oracle. 
    For example, if the resulting digit in a synthesized image is deemed an ``8" to an Oracle and the label assigned to the sample, by \mixup~or \nmixup, is also ``8", then the CE is low and the label is correct. 
    We also note that for the Oracle, the certainty of the predictions is correlated with the correctness of label. Finally, to address the issue of what Oracle to use, we adopt a highly accurate LeNet-5~\cite{lecun1998gradient} MNIST digit classifier that achieves $99.31\%$ classification accuracy on the standardized MNIST test set.
    
    \fig\ref{fig:mnist_vis} (f) and (g) show the quantitative results for the realism ($\propto$ 1/$\mathcal{H}$) of \mixup~and \nmixup's outputs, and the correctness of the corresponding labels ($\propto$ 1/CE) as evaluated by the Oracle, respectively, using kernel density estimate (KDE) plots with normalized areas. For both metrics, lower values (along the horizontal axes) are better. In \fig\ref{fig:mnist_vis} (f), we observe the \nmixup~has a higher peak for low values of entropy as compared to \mixup, indicating that the former generates more realistic samples. The inset figure therein shows the same plot with a logarithmic scale for the density, and \nmixup's improvements over \mixup~for higher values of entropy are clearly discernible here. Similarly, in \fig\ref{fig:mnist_vis} (g), we see that the cross entropy values for \nmixup~are concentrated around 0, whereas those for \mixup~are spread out more widely, implying that the former produces fewer samples with label error. If we restrict our samples to only those whose entropy of Oracle's predictions was less than $0.1$, meaning they were highly realistic samples, the label correctness distribution remains similar as shown in the inset figure, \ie  \mixup's outputs that look realistic are more likely to exhibit label error.

\begin{table*}[ht]
\centering
\caption{Classification performance (AUC and ACC) evaluated on MEDMNIST.
}
\label{tab:MedMNIST_results}
\resizebox{0.9\textwidth}{!}{%
\setlength{\tabcolsep}{0.5em}
\begin{tabular}{ccccccccc}
\toprule
\textbf{Dataset} & \multicolumn{2}{c}{PathMNIST} & \multicolumn{2}{c}{DermaMNIST} & \multicolumn{2}{c}{OCTMNIST} & \multicolumn{2}{c}{PneumoniaMNIST} \\ 
{\#images (\#classes)} & \multicolumn{2}{c}{107,180 (9)} & \multicolumn{2}{c}{10,005 (7)} & \multicolumn{2}{c}{109,309 (4)} & \multicolumn{2}{c}{5856 (2)} \\ 
\textbf{Method} & AUC & ACC & AUC & ACC & AUC & ACC & AUC & ACC \\ \hline
ERM & 0.962 & 84.4\% & 0.899 & 72.1\% & \textbf{0.951} & 70.8\% & 0.947 & 80.3\% \\ 
\mixup & 0.959 & 77.5\% & 0.897 & 72.2\% & 0.945 & 70.5\% & 0.945 & 75.4\% \\ 
\nmixup~($\pvar=2.8$) & \textbf{0.969} & \textbf{87.6\%} & \textbf{0.911} & \textbf{73.3\%} & 0.918 & \textbf{72.8\%} & \textbf{0.951} & \textbf{80.9\%} \\ 
\midrule
\textbf{Dataset} & \multicolumn{2}{c}{BreastMNIST} & \multicolumn{2}{c}{OrganMNIST\_A} & \multicolumn{2}{c}{OrganMNIST\_C} & \multicolumn{2}{c}{OrganMNIST\_S} \\ 
 {\#images (\#classes)} & \multicolumn{2}{c}{780 (2)} & \multicolumn{2}{c}{58,850 (11)} & \multicolumn{2}{c}{23,660 (11)} & \multicolumn{2}{c}{25,221 (11)} \\ 
\textbf{Method} & AUC & ACC & AUC & ACC & AUC & ACC & AUC & ACC \\ \hline
ERM & 0.897 & 85.9\% & 0.995 & 92.1\% & 0.990 & 88.9\% & 0.967 & 76.2\% \\ 
\mixup & 0.914 & 76.2\% & 0.995 & \textbf{93.1\%} & 0.990 & 89.9\% & 0.966 & 72.7\% \\ 
\nmixup~($\pvar=2.8$) & \textbf{0.928} & \textbf{87.2\%} & \textbf{0.996} & 92.7\% & \textbf{0.991} & \textbf{91.0\%} & \textbf{0.969} & \textbf{77.1\%} \\ 
\bottomrule
\end{tabular}%
}
\end{table*}

\subsection{Diversity}  \label{subsec:diversity}
We can control the diversity of \nmixup's output by changing $N$, \ie the number of points used as input to \nmixup, and the hyperparameter $\pvar$. As the value of $\pvar$ increases, the resulting distribution of the sampled points approaches the original data distribution. For example, in \fig\ref{fig:distribution_vis} (a), we see that changing $\pvar$ leads to an interpolation between \mixup-like and the original input-like distributions. Similarly, in \fig\ref{fig:distribution_vis} (c), we can see the effects of varying the batch size $N$ (\ie the number of input samples used to synthesize new samples) and $\pvar$. As $N$ increases, more original samples are used to generate the synthetic samples, and therefore the synthesized samples allow for a wider exploration of the space around the original samples. This effect is more pronounced with smaller values of $\pvar$ because with the weight assigned to one point, while still dominating all other weights, is not large enough to pull the synthetic sample close to it. This, along with fewer points to compute the weighted average of, leads to samples being generated farther from the original distribution as $\pvar$ decreases. On the other hand, as $\pvar$ increases, the contribution of one sample gets progressively larger, and as a result, the effect of a large $\pvar$ overshadows the effect of $N$.

\subsection{Richness of Labels}
The third desirable property of synthetic data is that, not only the generated samples should be able to capture and reflect the diversity of the original dataset, but also build upon it and extend it. As discussed in Sec.~\ref{sec:method}, for a single value of $\lambda$, \mixup~ generates 1 synthetic sample for every pair of original samples. In contrast, given a single value of $\pvar$ and $N$ original samples, \nmixup~can generate $N!$ new samples. The richness of the generated labels in \nmixup~comes from the fact that, unlike \mixup~whose outputs lie anywhere on the straight line between the original 2 samples, \nmixup~generates samples which are close to the original samples (as discussed in ``Realism" above) while still incorporating information from the original $N$ samples. As a case in point, consider the visualization of the soft labels in \mixup's and \nmixup's outputs on the MNIST dataset. Examining  \fig\ref{fig:mnist_vis} (b,d) again, we note \mixup's outputs are only made up of inputs from at most 2 classes. On the other hand, because of \nmixup's formulation, the outputs of \nmixup~can be made up of inputs from up to $min\left(N, \mathcal{K}\right)$ classes. This can also be seen in \nmixup's outputs in \fig\ref{fig:mnist_vis} (e): while the probability of one class dominates all others (see Theorem~\ref{theorem:pvalmin}), inputs from multiple classes, in addition to the dominant class, contribute to the final output and therefore this is reflected in the soft labels, leading to richer labels with information from multiple classes in 1 synthetic sample, which in turn arguably allow models trained on these samples to better learn the class decision boundaries.

\subsection{Preserving the Intrinsic Dimensionality of the Original Data}

As a direct consequence of the realism of synthetic data discussed above and its relation to the data manifold, we evaluate how the intrinsic dimensionality (ID hereafter) of the datasets change when \mixup~and \nmixup~are applied. 
With our 3D manifold visualizations in \fig\ref{fig:distribution_vis} (b), we saw that \mixup~samples points off the data manifold while \nmixup~limits the exploration of the high dimensional space, thus maintaining a lower ID. In order to substantiate this claim with quantitative results, we estimate the IDs of several datasets, both synthetic and real-world, and compare how the IDs of \mixup-~and \nmixup-generated distributions compare to those of the respective original distributions. For synthetic data, we use the high dimensional datasets described in Sec.~\ref{subsec:synth_data}, \ie 1-D helical manifolds embedded in $\mathbb{R}^3$ and in $\mathbb{R}^{12}$. For real-world datasets, we use the entire training partitions (50,000 images) of CIFAR-10 and CIFAR-100 datasets.
For each point in all the 4 datasets, the local measure of the ID (local ID hereafter) is calculated using a $k$-nearest neighborhood around each point with $k=8$ and $k=128$~\cite{bac2021scikit,fukunaga1971algorithm}. The means and the standard deviations of the local ID estimates for all the datasets: original data distribution, \mixup's output, and \nmixup's outputs for $\pvar \in [0, 15]$, are visualized in \fig\ref{fig:localID}.

The results in \fig\ref{fig:localID} support the observations from the discussion around the realism (Sec.~\ref{subsec:realism}) and the diversity (Sec.~\ref{subsec:diversity}) of outputs. In particular, notice how \mixup's off-manifold sampling leads to an inflated estimate of the local ID, whereas the local ID of \nmixup's output is lower than that of \mixup~and, as expected, can be controlled using $\pvar$. This difference is even more apparent with real-world high dimensional (3072-D) datasets, \ie CIFAR-10 and CIFAR-100, where for all values of $\pvar \ge \pvarmin$ (Theorem \ref{theorem:pvalmin}), as $\pvar$ increases, the local ID of \nmixup's output drops dramatically, meaning the resulting distributions lie on progressively lower dimensional intrinsic manifolds.

\subsection{Evaluation on Downstream Task: Classification}  \label{subsec:downstream}

Table~\ref{tab:natural_resultsV2} contains the performance evaluation of models trained using traditional data augmentation techniques, \eg rotation, flipping, and cropping, (``ERM"), and \mixup's and \nmixup's outputs from natural image datasets. For \nmixup, we choose 3 values of $\pvar$: $2.4$ (to allow exploration of the space around the original data manifold), $4.0$ (to restrict the synthetic samples to be close to the original samples), and $2.8$ (to allow for a behavior that permits exploration while still restricting the points to a small region around the original distribution). We see that 17 of the 18 models in Table~\ref{tab:natural_resultsV2} trained with \nmixup~outperform their ERM and \mixup~counterparts, with the lone exception being a model that is as accurate as \mixup. Next, Table~\ref{tab:skin_results_narrow} shows the performance of the models on the 10 skin lesion image diagnosis datasets 
($\pvar=\{2.4, 2.8, 4.0\}$). For both ResNet-18 and ResNet-50 and for all the 10 SKIN datasets, \nmixup~outperforms both \mixup~and ERM on skin lesion diagnosis tasks. Finally, Table~\ref{tab:MedMNIST_results} presents the quantitative evaluation on the 8 classification datasets from the MedMNIST collection, but use \nmixup~only with $\pvar=2.8$. In 6 out of the 8 datasets, \nmixup~outperforms both \mixup~and ERM, and in the other 2, \nmixup~achieves the highest value for 1 metric out of 2 each.

Note that these selected values of $\pvar$ can be changed to other reasonable values (please see the Appendix for sensitivity analysis of $\pvar$), and as shown above qualitatively and quantitatively, the desirable properties of \nmixup~hold for all values of $\pvar \geq \pvarmin$. Consequently, our quantitative results on classification tasks on 24 datasets show that \nmixup~outperforms ERM and \mixup~for all the datasets and in most cases, using all the 3 selected values of $\pvar$.

\subsection{Computational Efficiency}

\nmixup's PyTorch~\cite{paszke2019pytorch} implementation is provided in the Appendix. Our benchmarking experiments (Appendix) show that training DNNs for downstream tasks (Sec.~\ref{subsec:downstream}) with \nmixup~is at least as fast as \mixup, and for augmenting batches of 32 RGB images of $224 \times 224$ resolution, \nmixup~is over $2\times$ faster than \mixup.

\section{Conclusion}
We proposed \nmixup, a multi-sample generalization of the popular \mixup~technique for data augmentation that uses the terms of a truncated Riemann zeta function to combine $N\ge2$ samples of original dataset. We presented theoretical proofs that \mixup~is a special case of \nmixup~(when $N$=2 and with a specific setting of \nmixup's hyperparameter $\pvar$) and that the \nmixup~formulation allows for the weight assigned to one sample to dominate all the others, thus ensuring the synthesized samples are on or close to the original data manifold. The latter property leads to generating samples that are more realistic and, along with allowing $N > 2$, generates more diverse samples with richer labels as compared to their \mixup~counterparts.
We presented extensive experimental evaluation on controlled synthetic (1-D manifolds in 2-D and 3-D; 3-D manifolds in 12-D) and 24 real-world (natural and medical) image datasets of various modalities. We demonstrated quantitatively that, compared to \mixup: \nmixup~better preserves the intrinsic dimensionality of the original datasets; provides higher levels of realism and label correctness; and achieves stronger performance (i.e., higher accuracy) on multiple downstream classification tasks. Future work will include exploring \nmixup~in the learned feature space, although opinions on the theoretical justifications for interpolating in the latent space are not yet converged~\cite{cho2021manifold}.





{
\small

\bibliographystyle{plain}
\bibliography{main}



}

\clearpage
\pagenumbering{arabic}
\setcounter{table}{0}
\setcounter{figure}{0}
\renewcommand{\thetable}{A\arabic{table}}
\renewcommand{\thefigure}{A\arabic{figure}}
\renewcommand{\thelstlisting}{A\arabic{lstlisting}}
\appendix

\section*{\Large{Appendix}}

\textbf{List of contents:}

\begin{itemize}
    \item \textbf{Appendix~\ref{appndx:proofs}}: Proofs of Theorems~\ref{theorem:pvalmin} and \ref{theorem:special_case}.
    \item \textbf{Appendix~\ref{appndx:implementation}}: PyTorch implementation and benchmarking of \nmixup.
    \item \textbf{Appendix~\ref{appndx:ID_estimation}}: Details about local intrinsic dimensionality estimation.
    \item \textbf{Appendix~\ref{appndx:training_details}}: Details about datasets and model training for classification tasks.
    \item \textbf{Appendix~\ref{appndx:hyperparam_exps}}: Hyperparameter sensitivity analysis results on CIFAR-10 and CIFAR-100.
    \item \textbf{Appendix~\ref{appndx:skin_detailed_results}}: Detailed quantitative results on SKIN.
\end{itemize}

\section{Theorem Proofs}    \label{appndx:proofs}

\begin{thm}
For $\gamma \geq \pvarmin = \pvalmin$, the weight assigned to one sample dominates all other weights, i.e., $\forall \ \gamma \geq \pvalmin$,
\begin{equation}
    w_1 > \sum_{i=2}^N w_i
\end{equation}
\end{thm}

\begin{proof}
Let us consider the case when $N \to \infty$. We need to find the value of $\pvar$ such that
\begin{align}
    w_1 &> \sum_{i=2}^\infty w_i \\
    \Rightarrow \frac{1^{-\pvar}}{C} &> \sum_{i=2}^\infty \frac{i^{-\pvar}}{C} ; \ \ C = \sum_{j=1}^\infty j^{-\pvar} \\
    \Rightarrow 1^{-\pvar} &> \sum_{i=2}^\infty i^{-\pvar} \ (\mathrm{since } \ C > 0) \\
    \Rightarrow 1^{-\pvar} + 1^{-\pvar} &> 1^{-\pvar} + \sum_{i=2}^\infty i^{-\pvar} \\
    \Rightarrow 2 &> \sum_{i=1}^\infty i^{-\pvar}
\end{align}

Note that $\sum_{i=1}^\infty i^{-\pvar} = \zeta(\pvar)$ is the Riemann zeta function at $\pvar$. Using a solver, we get $\pvar \geq \pvalmin$. Therefore, $\forall \ \gamma \geq \pvarmin = \pvalmin$,
\begin{align}
    w_1 > \sum_{i=2}^\infty w_i > \sum_{i=2}^N w_i \Rightarrow w_1 > \sum_{i=2}^N w_i.
\end{align}
\end{proof}

\begin{thm}
For $N = 2$ and $\pvar= \log_2 \left(\frac{\lambda}{1-\lambda}\right)$, \mixup~simplifies to \nmixup.
\end{thm}

\begin{proof}
When $N=2$, \nmixup~(\eqn\ref{eqn:nmixup}) generates new samples by
\begin{equation}
    \begin{aligned}
    x_k = \sum_{i=1}^2 w_i x_i = w_1 x_1 + w_2 x_2 \\
    y_k = \sum_{i=1}^2 w_i y_i = w_1 y_1 + w_2 y_2,
    \end{aligned}
\end{equation}

\noindent where

\begin{equation}
    w_1 = \frac{1^{-\pvar}}{1^{-\pvar} + 2^{-\pvar}}; \ \ w_2 = \frac{2^{-\pvar}}{1^{-\pvar} + 2^{-\pvar}}.
\end{equation}

For this to be equivalent to \mixup~(\eqn\ref{eqn:mixup}), we should have

\begin{equation}
    w_1 = \lambda; \ \ w_2 = 1 - \lambda.
\end{equation}

Solving for $\pvar$, we have

\begin{align}
    w_1 &= \frac{1^{-\pvar}}{1^{-\pvar} + 2^{-\pvar}} = \lambda \\
    &\Rightarrow \frac{1}{1 + 2^{-\pvar}} = \lambda \\
    &\Rightarrow 2^{-\pvar} = \frac{1-\lambda}{\lambda} \\
    &\Rightarrow \pvar = -  \log_2 \left(\frac{1-\lambda}{\lambda}\right) = \log_2 \left(\frac{\lambda}{1-\lambda}\right).
\end{align}
\end{proof}






\section{\nmixup: Implementation and Benchmarking}  \label{appndx:implementation}


The \nmixup~implementation in PyTorch~\cite{paszke2019pytorch} is shown in Listing~\ref{listing:code} and in the \texttt{Appendix\_utils.py} file. Unlike \mixup~which performs scalar multiplications of $\lambda$ and $1-\lambda$ with the input batches, \nmixup~performs a single matrix multiplication of the input batches with the weights. With our optimized implementation, we find that model training times using \nmixup~are as fast as, if not faster than, those using \mixup~when evaluated on datasets with different spatial resolutions: CIFAR-10 ($32 \times 32$ RGB images), STL-10 ($96 \times 96$ RGB images), and Imagenette ($224 \times 224$ RGB images), as shown in Table~\ref{tab:benchmarks}. Moreover, when using \mixup~and \nmixup~on a batch of 32 tensors of $224 \times 224$ spatial resolution with 3 feature channels, which is the case with popular ImageNet-like training regimes, \nmixup~is over twice as fast as \mixup~and over 110 times faster than the original local synthetic instances implementation~\cite{brown2015prediction}. 

All models were trained and benchmarked on a workstation with Intel Core i9-9900K and 32 GB of memory with the Nvidia GeForce GTX TITAN X GPU with 12 GB of memory.

\begin{table*}[ht!]
\centering
\caption{Benchmarking \nmixup~against \mixup~for training models on CIFAR-10, STL-10, Imagenette, and for augmenting a batch of 32 RGB images of $224 \times 224$ spatial resolution.}
\label{tab:benchmarks}
\resizebox{0.9\textwidth}{!}{%
\setlength{\tabcolsep}{0.5em}
\def\arraystretch{1.5}
\begin{tabular}{cccccc}
\toprule
\multicolumn{2}{c}{Method}                                                                         & \begin{tabular}[c]{@{}c@{}}\textbf{CIFAR-10}\\ (200 epochs)\end{tabular} & \begin{tabular}[c]{@{}c@{}}\textbf{STL-10}\\ (200 epochs)\end{tabular} & \begin{tabular}[c]{@{}c@{}}\textbf{Imagenette}\\ (80 epochs)\end{tabular} & \begin{tabular}[c]{@{}c@{}}[32, 3, 224, 224]\\\texttt{torch.Tensor}\end{tabular} \\ \midrule 
\parbox[t]{2mm}{\multirow{3}{*}{\rotatebox[origin=c]{90}{Wall Time}}} & \mixup  & 1h 19m $\pm$ 30s & 24m 59s $\pm$ 16.9s & 45m 39s $\pm$ 8.5s  & 745$\mu$s $\pm$ 9.55$\mu$s                                               \\ 
\multicolumn{1}{c}{}                                                                     & \nmixup & 1h 21m $\pm$ 23s & 24m 58s $\pm$ 4.6s  & 45m 34s $\pm$ 14.1s & 345$\mu$s $\pm$ 2.53$\mu$s                                               \\ \cdashline{2-6} 
\multicolumn{1}{c}{}                                                                     & 
\begin{tabular}[c]{@{}c@{}}Local synthetic\\instances~\cite{brown2015prediction}\end{tabular}     & -                & -                   & -                   & 38.7ms $\pm$ 1.33 ms                                                     \\ 
\bottomrule
\end{tabular}%
}
\end{table*}

\begin{center}
\begin{minipage}[h!]{.9\linewidth}
\begin{lstlisting}[language=Python,caption=PyTorch-style implementation of \nmixup.,label=listing:code,captionpos=t,linewidth=10.3cm]
import torch.nn.functional as F

def zeta_mixup(X, Y, n_classes, weights):
    """
    X -> input feature tensor ([N, C, H, W])
    Y -> label tensor ([N, 1])
    weights -> weights tensor ([W, W])
    N: batch size; C: channels; H: height; W: width
    """
    # compute weighted average of all samples
    X_new = torch.einsum("ijkl,pi->pjkl", X, weights)
    
    # encode original labels to one-hot vectors
    Y_onehot = F.one_hot(Y, n_classes)
    # compute weighted average of all labels
    Y_new = torch.einsum("pq,qj->pj", weights, Y_onehot)
    
    # return synthesized samples and labels
    return X_new, Y_new

# Specify number of classes and training batch size
n_cls, b_size = 10, 32
    
# Random training batch constructed for illustration
x = torch.randn(b_size, 3, 224, 224).cuda()
y = torch.randint(0, (n_cls-1), (b_size,)).cuda()

# Generate weights using normalized p-series
weights = zeta_mixup_weights(batch_size=b_size).cuda()

# Perform zeta-mixup on the training batch
x_new, y_new = zeta_mixup(x, y, n_cls, weights)

\end{lstlisting}
\end{minipage}
\end{center}

\section{Intrinsic Dimensionality Estimation}   \label{appndx:ID_estimation}

While the ID of a dataset can be estimated globally, datasets can have heterogenous regions and thus consist of regions of varying IDs. As such, instead of a global estimate of the ID, a local measure of the ID (local ID hereafter), estimated in the local neighborhood of each point in the dataset with neighborhoods typically defined using the $k$-nearest neighbors, is more informative of the inherent organization of the dataset. For our local ID estimation experiments, we use a principal component analysis-based local ID estimator from the \texttt{scikit-dimension} Python library~\cite{bac2021scikit} using the Fukunaga-Olsen method~\cite{fukunaga1971algorithm}, where an eigenvalue is considered significant if it is larger than $5\%$ of the largest eigenvalue.

\section{Training Details for Classification Task Models}   \label{appndx:training_details}

\subsection{Natural Image Datasets}

MNIST and F-MNIST have $28 \times 28$ grayscale images. CIFAR-10 and CIFAR-100 datasets which have RGB images with $32 \times 32$ spatial resolution. STL-10 consists of RGB images with a higher $96 \times 96$ resolution and also has fewer training images than testing images per class. Released by Jeremy Howard to facilitate evaluation on natural images from the original ImageNet dataset~\cite{deng2009imagenet} but with more reasonable computational and time requirements, Imagenette and Imagewoof~\cite{howard2019imagenette} are 10-class subsets each of the ImageNet dataset. The list of ImageNet classes and the corresponding synset IDs from WordNet in both these datasets are shown in Table~\ref{tab:imagenet_subsets_classes}. Both the datasets have standardized training and validation partitions.




For all the 6 natural image datasets: CIFAR-10, CIFAR-100, F-MNIST, STL-10, Imagenette, and Imagewoof, we train and validate deep models with the ResNet-18 architecture~\cite{he2016deep} on the standard training and validation partitions and use random horizontal flipping for data augmentation.

For CIFAR-10, CIFAR-100, F-MNIST, and STL-10, the models are trained on the original image resolutions, whereas for Imagenette and Imagewoof, the images are resized to $224 \times 224$. For CIFAR-10, CIFAR-100, F-MNIST, the models are trained for 200 epochs with an initial learning rate of $0.1$, which is decayed by a multiplicative factor of $0.2$ at $80^{\mathrm{th}}$, $120^{\mathrm{th}}$, and $160^{\mathrm{th}}$ epochs, with batches of 128 images for CIFAR datasets and 32 images for F-MNIST. For STL-10, the models are trained for 120 epochs with a batch size of 32 and an initial learning rate of $0.1$, which is decayed by a multiplicative factor of $0.2$ at $80^{\mathrm{th}}$ epoch. Finally, for Imagenette and Imagewoof, the models are trained for 80 epochs with a batch size of 32 and an initial learning rate of $0.01$, which is decayed by a multiplicative factor of $0.2$ at $25^{\mathrm{th}}$, $50^{\mathrm{th}}$, and $65^{\mathrm{th}}$ epochs. All models are optimized using cross entropy loss and mini-batch stochastic gradient descent (SGD) with Nesterov momentum of $0.9$ and a weight decay of $5$e$-4$.

\begin{table*}[ht!]
\centering
\caption{List of classes from ImageNet and the corresponding WordNet synset IDs in Imagenette and Imagewoof datasets.}
\label{tab:imagenet_subsets_classes}
\resizebox{\textwidth}{!}{%
\setlength{\tabcolsep}{0.5em}
\def\arraystretch{1.5}
\begin{tabular}{cccccccccccc}
\toprule
{\multirow{2}{*}{\rotatebox[origin=c]{90}{\textbf{Imagenette}}}} & \begin{tabular}[c]{@{}c@{}}\textbf{ImageNet}\\\textbf{class}\end{tabular} & tench & \begin{tabular}[c]{@{}c@{}}English\\springer\end{tabular} & \begin{tabular}[c]{@{}c@{}}cassette\\player\end{tabular} & chain saw & church & \begin{tabular}[c]{@{}c@{}}French\\horn\end{tabular} & \begin{tabular}[c]{@{}c@{}}garbage\\truck\end{tabular} & gas pump & golf ball & parachute \\
 & \begin{tabular}[c]{@{}c@{}}\textbf{WordNet}\\\textbf{synset ID}\end{tabular} & n01440764 & n02102040 & n02979186 & n03000684 & n03028079 & n03394916 & n03417042 & n03425413 & n03445777 & n03888257 \\ 
 \midrule
{\multirow{2}{*}{\rotatebox[origin=c]{90}{\textbf{Imagewoof}}}} & \begin{tabular}[c]{@{}c@{}}\textbf{ImageNet}\\\textbf{class}\end{tabular} & \begin{tabular}[c]{@{}c@{}}Australian\\terrier\end{tabular} & \begin{tabular}[c]{@{}c@{}}Border\\terrier\end{tabular} & Samoyed & Beagle & Shih-Tzu & \begin{tabular}[c]{@{}c@{}}English\\foxhound\end{tabular} & \begin{tabular}[c]{@{}c@{}}Rhodesian\\ridgeback\end{tabular} & Dingo & \begin{tabular}[c]{@{}c@{}}Golden\\retriever\end{tabular} & \begin{tabular}[c]{@{}c@{}}Old English\\sheepdog\end{tabular} \\
 & \begin{tabular}[c]{@{}c@{}}\textbf{WordNet}\\\textbf{synset ID}\end{tabular} & n02096294 & n02093754 & n02111889 & n02088364 & n02086240 & n02089973 & n02087394 & n02115641 & n02099601 & n02105641 \\ 
 \bottomrule
\end{tabular}%
}
\end{table*}

\subsection{Skin Lesion Image Diagnosis Datasets}

Skin lesion imaging has 2 pre-dominant modalities: clinical images and dermoscopic images. While both capture RGB images, clinical images consist of close-up lesion images acquired with consumer-grade cameras, whereas dermoscopic images are acquired using a dermatoscope which allows for identification of detailed morphological structures~\cite{menzies2009dermoscopy} along with fewer imaging-related artifacts~\cite{kittler2002diagnostic}. 

For all the 10 skin lesion image diagnosis datasets, we train classification models with the ResNet-18 and the ResNet-50 architectures. For data augmentation, we take a square center-crop of the image with edge length equal to 0.8*\textit{min}(height, width) and then resize it to $256 \times 256$ spatial resolution. The ISIC 2016, 2017, and 2018 come with standardized partitions that we use for training and evaluating our models, and for the other 7 datasets, we perform a stratified split in the ratio of training : validation : testing :: $70:10:20$. For all the datasets, we use the 5-class diagnosis labels used in the original dataset paper and in the literature~\cite{kawahara2018seven,coppola2020interpreting,abhishek2021predicting}: ``basal cell carcinoma", ``nevus", ``melanoma", ``seborrheic keratosis", and ``others".

For all the datasets except ISIC 2018, we use a batch size of 32 images and train the models for 50 epochs with an initial learning rate of $0.01$, which was decayed by a multiplicative factor of $0.1$ every $10$ epochs. Given that the ISIC 2018 dataset is considerably larger, we train it for 20 epochs with 32 images in a batch and an initial learning rate of $0.01$, which was decayed by a multiplicative factor of $0.1$ every $4$ epochs. As with experiments with the natural image datasets, all models are optimized using cross entropy loss and SGD with Nesterov momentum of $0.9$ and a weight decay of $5$e$-4$.

\subsection{Datasets of Other Medical Imaging Modalities from MedMNIST}

For all the 8 datasets from the MedMNIST collection, we train and evaluate classification models with the ResNet-18 architecture on the standard training, validation, and testing partitions. The images are used in their original $28 \times 28$ spatial resolution.

For all the datasets, we use a learning rate of $0.01$ and following the original paper~\cite{yang2021medmnist}, we use cross entropy loss with SGD on batches of 128 images to optimize the classification models.

\section{\nmixup: Hyperparameter Sensitivity Analysis}   \label{appndx:hyperparam_exps}

\begin{figure*}[ht!]
     \centering
     \begin{subfigure}[t]{0.45\textwidth}
         \centering
         \includegraphics[width=\textwidth]{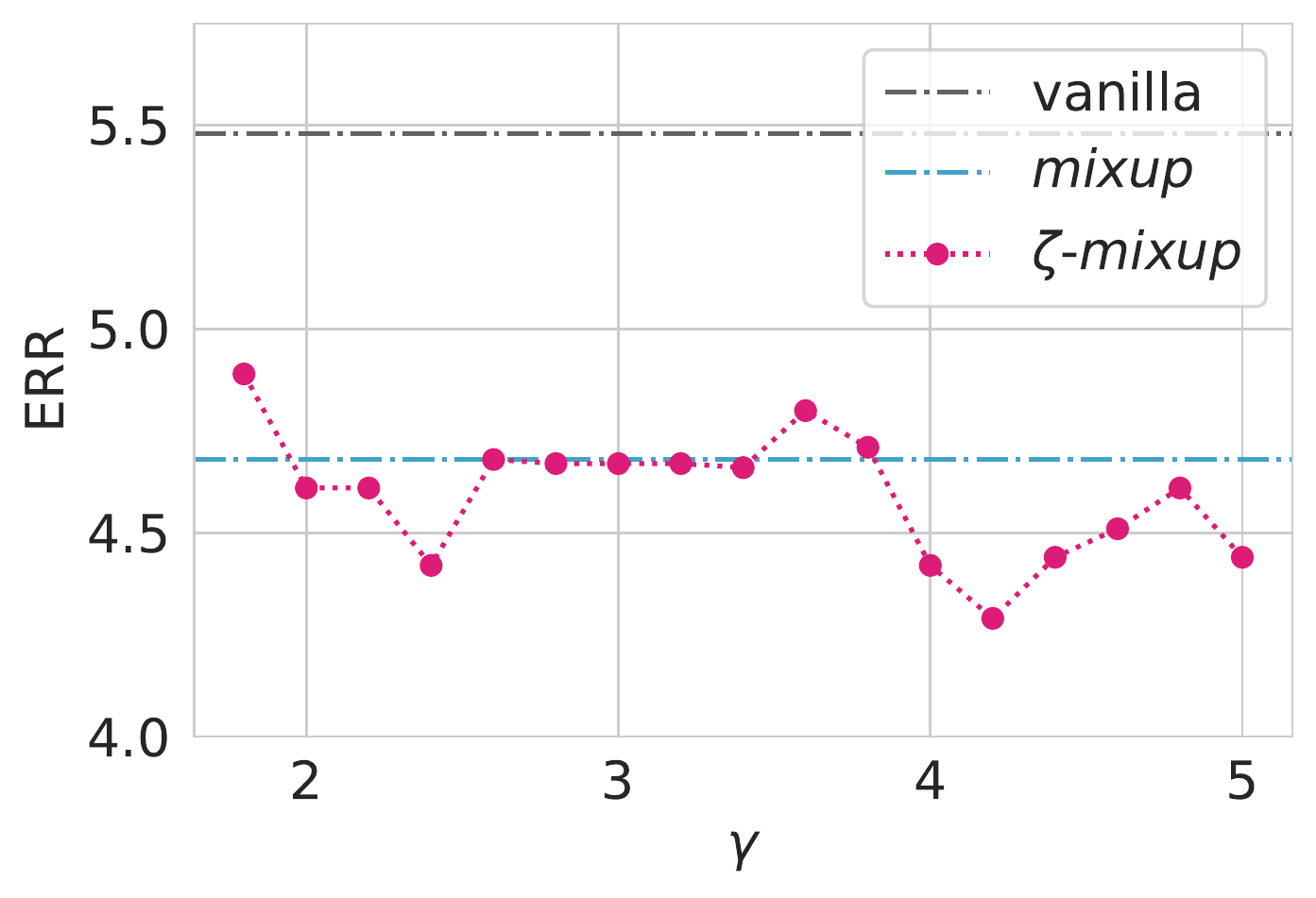}
         \caption{$\pvar$-sensitivity analysis (CIFAR-10.)}
     \end{subfigure}
     \hfill
     \begin{subfigure}[t]{0.45\textwidth}
         \centering
         \includegraphics[width=\textwidth]{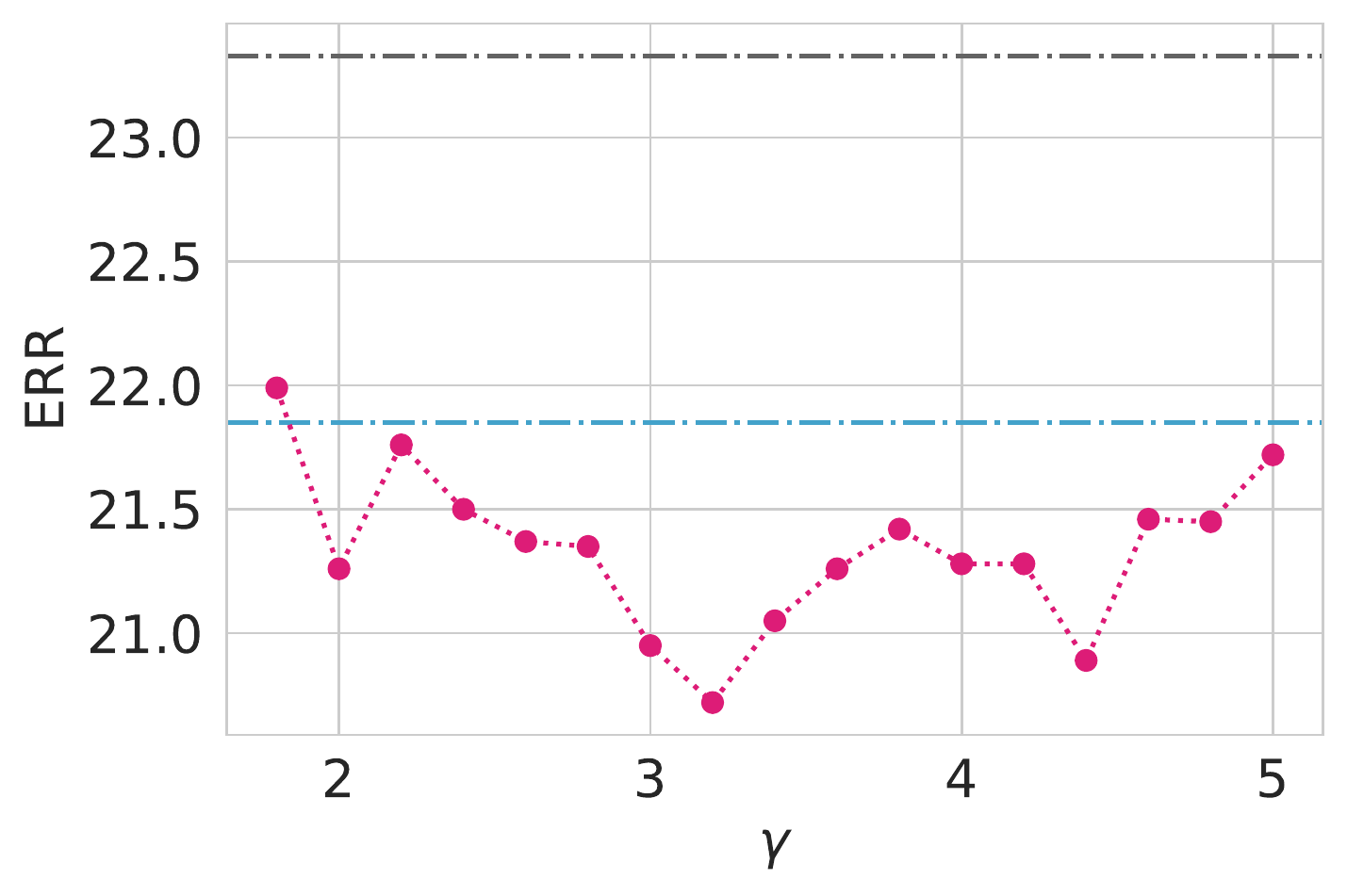}
         \caption{$\pvar$-sensitivity analysis (CIFAR-100).}
     \end{subfigure}
     \\
     \begin{subfigure}[t]{0.8\textwidth}
         \centering
         \includegraphics[width=\textwidth]{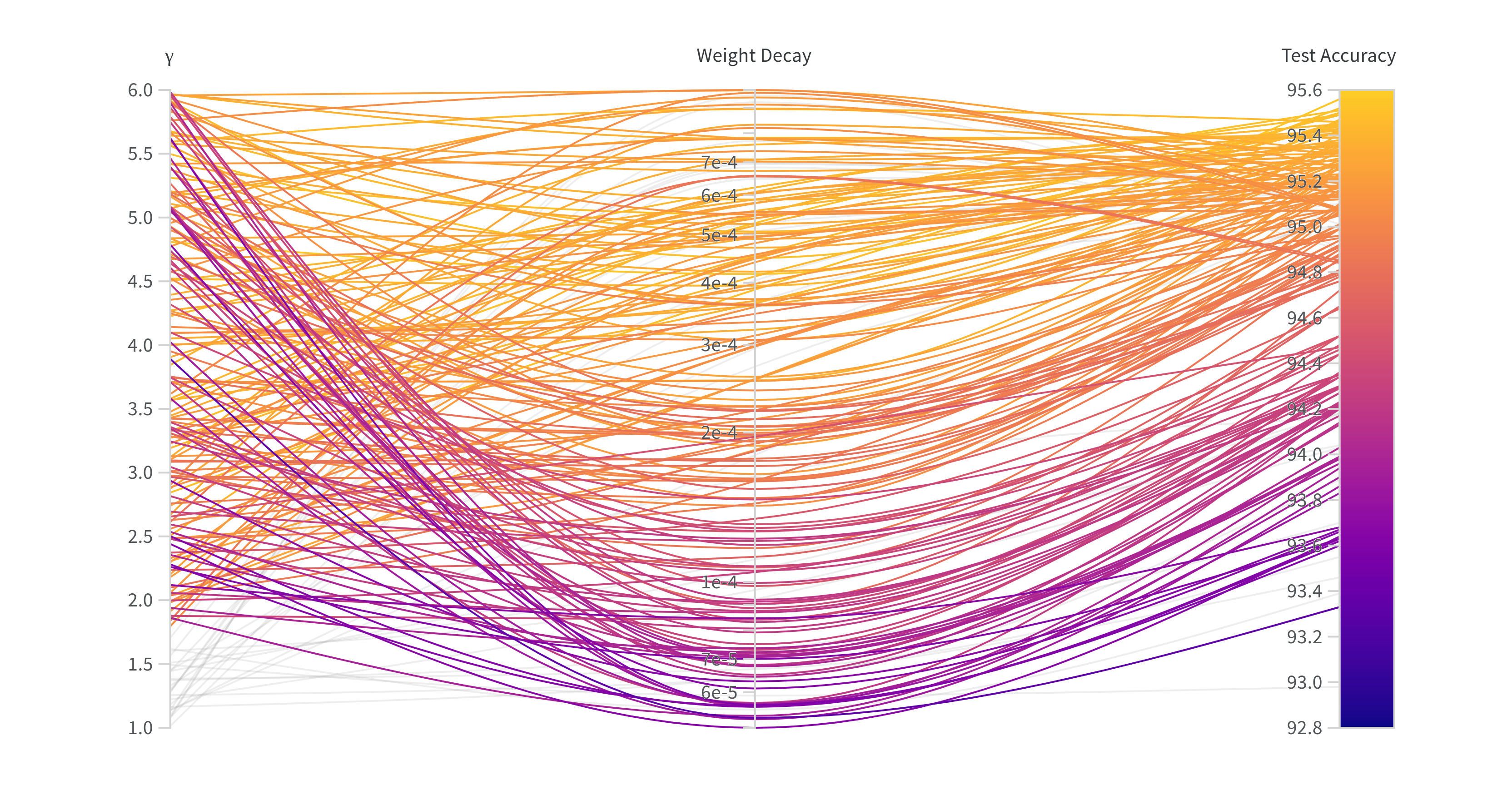}
         \caption{Hyperparameter sweeps for $\pvar$ and weight decay (CIFAR-10).}
     \end{subfigure}
     \\
      \begin{subfigure}[t]{0.8\textwidth}
         \centering
         \includegraphics[width=\textwidth]{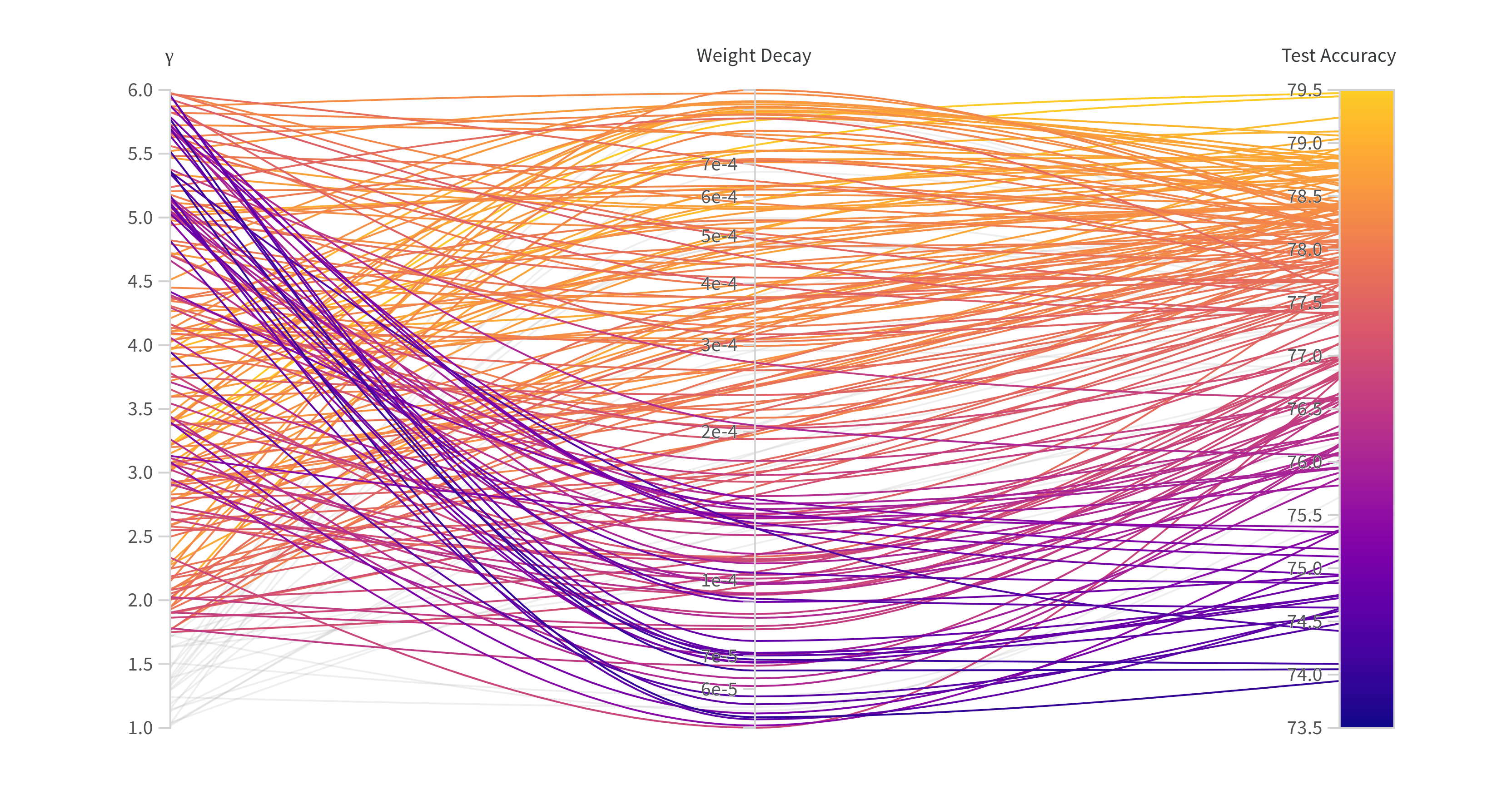}
         \caption{Hyperparameter sweeps for $\pvar$ and weight decay (CIFAR-100).}
     \end{subfigure}
        \caption{Hyperparameter sensitivity analysis for \nmixup~on CIFAR-10 and CIFAR-100. In (a, b), $\pvar$ is varied from $[1.8, 5.0]$ and the resulting ERR is shown. In (c, d), 200 models are trained by varying $\pvar$ uniformly in $[1.0, 6.0]$ and weight decay log-uniformly in $[5$e$-5, 1$e$-3]$ and the resulting test accuracies are plotted on the right-most column. Models with $\pvar < \pvarmin$ are shown in light gray.}
        \label{fig:sweep_vis}
\end{figure*}

We conduct extensive experiments on CIFAR-10 and CIFAR-100 datasets to analyze the effect of hyperparameters, particularly $\pvar$ and weight decay, on the performance of \nmixup.

First, we vary the hyperparameter $\pvar$ by choosing values from [1.8, 2.0, 2.2, $\cdots$, 5.0] and train and evaluate ResNet-18 models on CIFAR-10 and CIFAR-100. The corresponding overall error rates (ERR) are shown in \fig\ref{fig:sweep_vis} (a) and (b) respectively.

Next, we perform a hyperparameter sweep by varying both $\pvar$ and weight decay and training and evaluating ResNet-18 models on CIFAR-10 and CIFAR-100. We use {\ul Weights and Biases}\footnote{L. Biewald, “Experiment Tracking with Weights and Biases,” Weights \& Biases. [Online]. Available: \url{http://wandb.com/}. [Accessed: February 18, 2022].} to perform a Bayesian search over this hyperparameter space and sample $\pvar$ from a uniform distribution over [1.0, 6.0] and weight decay from a log-uniform distribution over [5e-5, 1e-3]. We perform 200 sweeps, effectively training 200 models, for both CIFAR-10 and CIFAR-100 each and plot the overall test accuracy in \fig\ref{fig:sweep_vis} (c) and (d) respectively. Models trained with $\pvar < \pvarmin$ are shown in light gray.

\section{Detailed Quantitative Results on Skin Lesion Diagnosis Datasets}   \label{appndx:skin_detailed_results}

\begin{table*}[ht!]
\centering
\caption{Classification performance evaluated on SKIN. \textsuperscript{$\dagger$} and \textsuperscript{$\ddagger$} denote dermoscopic and clinical skin lesion images respectively. The evaluation metrics are 
ACC$_{\mathrm{bal}}$, F1-micro, and F1-macro.
The highest and the second highest values of each metric have been formatted with \textbf{bold} and {\ul underline} respectively.}
\label{tab:skin_results}
\resizebox{\textwidth}{!}{%
\setlength{\tabcolsep}{0.35em}
\def\arraystretch{1.65}
\begin{tabular}{c|cccccc|cccccc}
\toprule
\textbf{Dataset}              & \multicolumn{6}{c|}{ISIC 2016}                                                                                                                                                             & \multicolumn{6}{c}{ISIC 2017}                                                                                                                                                                                       \\ 
\#images (\#classes) & \multicolumn{6}{c|}{1,279 (2)}                                                                                                                                                                                        & \multicolumn{6}{c}{2,750 (3)}                                                                                                                                                                                        \\ 
Method               & \multicolumn{3}{c:}{ResNet-18}                                                                                      & \multicolumn{3}{c|}{ResNet-50}                                                                 & \multicolumn{3}{c:}{ResNet-18}                                                                                      & \multicolumn{3}{c}{ResNet-50}                                                                 \\ 
                     & \multicolumn{1}{c}{ACC$_{\mathrm{bal}}$}        & \multicolumn{1}{c}{F1-micro}        & \multicolumn{1}{c}{F1-macro}        & \multicolumn{1}{c}{ACC$_{\mathrm{bal}}$}        & \multicolumn{1}{c}{F1-micro}        & F1-macro        & \multicolumn{1}{c}{ACC$_{\mathrm{bal}}$}        & \multicolumn{1}{c}{F1-micro}        & \multicolumn{1}{c}{F1-macro}        & \multicolumn{1}{c}{ACC$_{\mathrm{bal}}$}        & \multicolumn{1}{c}{F1-micro}        & F1-macro        \\ \hdashline 
ERM                  & \multicolumn{1}{c}{70.44\%}          & \multicolumn{1}{c}{0.7836}          & \multicolumn{1}{c:}{0.6865}          & \multicolumn{1}{c}{71.75\%}          & \multicolumn{1}{c}{0.8127}          & 0.7121          & \multicolumn{1}{c}{69.31\%}          & \multicolumn{1}{c}{0.7383}          & \multicolumn{1}{c:}{0.6720}          & \multicolumn{1}{c}{68.20\%}          & \multicolumn{1}{c}{0.6867}          & 0.6361          \\ 
\mixup~               & \multicolumn{1}{c}{71.77\%}          & \multicolumn{1}{c}{0.7968}          & \multicolumn{1}{c:}{0.7017}          & \multicolumn{1}{c}{72.08\%}          & \multicolumn{1}{c}{0.8179}          & 0.7175          & \multicolumn{1}{c}{71.60\%}          & \multicolumn{1}{c}{0.7333}          & \multicolumn{1}{c:}{0.6756}          & \multicolumn{1}{c}{71.51\%}          & \multicolumn{1}{c}{0.7433}          & 0.6979          \\ 
\nmixup~(2.4)        & \multicolumn{1}{c}{\textbf{74.53\%}} & \multicolumn{1}{c}{{\ul 0.8417}}    & \multicolumn{1}{c:}{{\ul 0.7180}}    & \multicolumn{1}{c}{71.52\%}          & \multicolumn{1}{c}{\textbf{0.8654}} & {\ul 0.7492}    & \multicolumn{1}{c}{\textbf{73.02\%}} & \multicolumn{1}{c}{0.7483}          & \multicolumn{1}{c:}{{\ul 0.6965}}    & \multicolumn{1}{c}{\textbf{72.91\%}} & \multicolumn{1}{c}{\textbf{0.7783}} & \textbf{0.7099} \\ 
\nmixup~(2.8)        & \multicolumn{1}{c}{{\ul 73.03\%}}    & \multicolumn{1}{c}{\textbf{0.8654}} & \multicolumn{1}{c:}{\textbf{0.7588}} & \multicolumn{1}{c}{\textbf{72.20\%}} & \multicolumn{1}{c}{{\ul 0.8602}}    & \textbf{0.7493} & \multicolumn{1}{c}{{\ul 72.33\%}}    & \multicolumn{1}{c}{\textbf{0.7633}} & \multicolumn{1}{c:}{\textbf{0.7068}} & \multicolumn{1}{c}{69.99\%}          & \multicolumn{1}{c}{{\ul 0.7733}}    & {\ul 0.7028}    \\ 
\nmixup~(4.0)        & \multicolumn{1}{c}{72.27\%}          & \multicolumn{1}{c}{0.7968}          & \multicolumn{1}{c:}{0.7043}          & \multicolumn{1}{c}{{\ul 72.11\%}}    & \multicolumn{1}{c}{0.8391}          & 0.7151          & \multicolumn{1}{c}{70.93\%}          & \multicolumn{1}{c}{{\ul 0.7567}}    & \multicolumn{1}{c:}{0.6815}          & \multicolumn{1}{c}{{\ul 72.39\%}}    & \multicolumn{1}{c}{0.7517}          & 0.6963          \\ 
\midrule

\textbf{Dataset}              & \multicolumn{6}{c|}{ISIC 2018}                                                                                                                                                             & \multicolumn{6}{c}{MSK}                                                                                                                                                                                             \\ 
\#images (\#classes) & \multicolumn{6}{c|}{10,015 (5)}                                                                                                                                                                                       & \multicolumn{6}{c}{3,551 (4)}                                                                                                                                                                                        \\ 
Method               & \multicolumn{3}{c:}{ResNet-18}                                                                                      & \multicolumn{3}{c|}{ResNet-50}                                                                 & \multicolumn{3}{c:}{ResNet-18}                                                                                      & \multicolumn{3}{c}{ResNet-50}                                                                 \\ 
                     & \multicolumn{1}{c}{ACC$_{\mathrm{bal}}$}        & \multicolumn{1}{c}{F1-micro}        & \multicolumn{1}{c:}{F1-macro}        & \multicolumn{1}{c}{ACC$_{\mathrm{bal}}$}        & \multicolumn{1}{c}{F1-micro}        & F1-macro        & \multicolumn{1}{c}{ACC$_{\mathrm{bal}}$}        & \multicolumn{1}{c}{F1-micro}        & \multicolumn{1}{c:}{F1-macro}        & \multicolumn{1}{c}{ACC$_{\mathrm{bal}}$}        & \multicolumn{1}{c}{F1-micro}        & F1-macro        \\ \hdashline 
ERM                  & \multicolumn{1}{c}{84.31\%}          & \multicolumn{1}{c}{0.8756}          & \multicolumn{1}{c:}{0.8122}          & \multicolumn{1}{c}{81.28\%}          & \multicolumn{1}{c}{0.8653}          & 0.7982          & \multicolumn{1}{c}{62.35\%}          & \multicolumn{1}{c}{0.6986}          & \multicolumn{1}{c:}{0.5999}          & \multicolumn{1}{c}{63.86\%}          & \multicolumn{1}{c}{0.7873}          & 0.6586          \\ 
\mixup~               & \multicolumn{1}{c}{83.96\%}          & \multicolumn{1}{c}{0.8394}          & \multicolumn{1}{c:}{0.7767}          & \multicolumn{1}{c}{85.65\%}          & \multicolumn{1}{c}{0.8601}          & 0.8064          & \multicolumn{1}{c}{63.59\%}          & \multicolumn{1}{c}{0.7423}          & \multicolumn{1}{c:}{0.6404}          & \multicolumn{1}{c}{{\ul 65.62\%}}    & \multicolumn{1}{c}{{\ul 0.7958}}    & 0.6434          \\ 
\nmixup~(2.4)        & \multicolumn{1}{c}{\textbf{87.20\%}} & \multicolumn{1}{c}{\textbf{0.8964}} & \multicolumn{1}{c:}{\textbf{0.8441}} & \multicolumn{1}{c}{84.75\%}          & \multicolumn{1}{c}{0.8653}          & 0.8112          & \multicolumn{1}{c}{\textbf{65.52\%}} & \multicolumn{1}{c}{{\ul 0.7746}}    & \multicolumn{1}{c:}{{\ul 0.6475}}    & \multicolumn{1}{c}{65.23\%}          & \multicolumn{1}{c}{\textbf{0.8056}} & \textbf{0.6875} \\ 
\nmixup~(2.8)        & \multicolumn{1}{c}{{\ul 84.67\%}}    & \multicolumn{1}{c}{0.8756}          & \multicolumn{1}{c:}{{\ul 0.8066}}    & \multicolumn{1}{c}{{\ul 86.59\%}}    & \multicolumn{1}{c}{{\ul 0.9016}}    & {\ul 0.8333}    & \multicolumn{1}{c}{{\ul 64.87\%}}    & \multicolumn{1}{c}{\textbf{0.7845}} & \multicolumn{1}{c:}{\textbf{0.6883}} & \multicolumn{1}{c}{\textbf{65.94\%}} & \multicolumn{1}{c}{0.7930}          & {\ul 0.6704}    \\ 
\nmixup~(4.0)        & \multicolumn{1}{c}{83.63\%}          & \multicolumn{1}{c}{{\ul 0.8808}}    & \multicolumn{1}{c:}{0.8062}          & \multicolumn{1}{c}{\textbf{89.18\%}} & \multicolumn{1}{c}{\textbf{0.9223}} & \textbf{0.8718} & \multicolumn{1}{c}{62.39\%}          & \multicolumn{1}{c}{0.6930}          & \multicolumn{1}{c:}{0.6006}          & \multicolumn{1}{c}{65.33\%}          & \multicolumn{1}{c}{0.7817}          & 0.6587          \\ 
\midrule

\textbf{Dataset}              & \multicolumn{6}{c|}{UDA}                                                                                                                                                                   & \multicolumn{6}{c}{DermoFit}                                                                                                                                                                                        \\ 
\#images (\#classes) & \multicolumn{6}{c|}{601 (2)}                                                                                                                                                                                         & \multicolumn{6}{c}{1,300 (5)}                                                                                                                                                                                        \\ 
Method               & \multicolumn{3}{c:}{ResNet-18}                                                                                      & \multicolumn{3}{c|}{ResNet-50}                                                                 & \multicolumn{3}{c:}{ResNet-18}                                                                                      & \multicolumn{3}{c}{ResNet-50}                                                                 \\ 
                     & \multicolumn{1}{c}{ACC$_{\mathrm{bal}}$}        & \multicolumn{1}{c}{F1-micro}        & \multicolumn{1}{c:}{F1-macro}        & \multicolumn{1}{c}{ACC$_{\mathrm{bal}}$}        & \multicolumn{1}{c}{F1-micro}        & F1-macro        & \multicolumn{1}{c}{ACC$_{\mathrm{bal}}$}        & \multicolumn{1}{c}{F1-micro}        & \multicolumn{1}{c:}{F1-macro}        & \multicolumn{1}{c}{ACC$_{\mathrm{bal}}$}        & \multicolumn{1}{c}{F1-micro}        & F1-macro        \\ \hdashline 
ERM                  & \multicolumn{1}{c}{67.46\%}          & \multicolumn{1}{c}{0.7000}          & \multicolumn{1}{c:}{0.6666}          & \multicolumn{1}{c}{66.85\%}          & \multicolumn{1}{c}{0.6917}          & 0.6593          & \multicolumn{1}{c}{80.43\%}          & \multicolumn{1}{c}{0.8269}          & \multicolumn{1}{c:}{0.8120}          & \multicolumn{1}{c}{83.24\%}          & \multicolumn{1}{c}{0.8500}          & 0.8316          \\ 
\mixup~               & \multicolumn{1}{c}{69.38\%}          & \multicolumn{1}{c}{0.7167}          & \multicolumn{1}{c:}{0.6851}          & \multicolumn{1}{c}{67.27\%}          & \multicolumn{1}{c}{0.7167}          & 0.6727          & \multicolumn{1}{c}{81.17\%}          & \multicolumn{1}{c}{0.8577}          & \multicolumn{1}{c:}{0.8302}          & \multicolumn{1}{c}{84.37\%}          & \multicolumn{1}{c}{0.8500}          & 0.8406          \\ 
\nmixup~(2.4)        & \multicolumn{1}{c}{\textbf{70.54\%}} & \multicolumn{1}{c}{\textbf{0.8000}} & \multicolumn{1}{c:}{\textbf{0.7272}} & \multicolumn{1}{c}{{\ul 68.39\%}}    & \multicolumn{1}{c}{0.7417}          & {\ul 0.6900}    & \multicolumn{1}{c}{82.57\%}          & \multicolumn{1}{c}{0.8692}          & \multicolumn{1}{c:}{0.8419}          & \multicolumn{1}{c}{{\ul 86.26\%}}    & \multicolumn{1}{c}{0.8615}          & 0.8491          \\ 
\nmixup~(2.8)        & \multicolumn{1}{c}{{\ul 70.22\%}}    & \multicolumn{1}{c}{{\ul 0.7667}}    & \multicolumn{1}{c:}{{\ul 0.7127}}    & \multicolumn{1}{c}{\textbf{70.92\%}} & \multicolumn{1}{c}{\textbf{0.7667}} & \textbf{0.7176} & \multicolumn{1}{c}{{\ul 83.50\%}}    & \multicolumn{1}{c}{{\ul 0.8731}}    & \multicolumn{1}{c:}{{\ul 0.8459}}    & \multicolumn{1}{c}{85.91\%}          & \multicolumn{1}{c}{{\ul 0.8962}}    & {\ul 0.8765}    \\ 
\nmixup~(4.0)        & \multicolumn{1}{c}{67.88\%}          & \multicolumn{1}{c}{0.7250}          & \multicolumn{1}{c:}{0.6800}          & \multicolumn{1}{c}{67.59\%}          & \multicolumn{1}{c}{{\ul 0.7500}}    & 0.6865          & \multicolumn{1}{c}{\textbf{83.94\%}} & \multicolumn{1}{c}{\textbf{0.8769}} & \multicolumn{1}{c:}{\textbf{0.8514}} & \multicolumn{1}{c}{\textbf{88.16\%}} & \multicolumn{1}{c}{\textbf{0.9115}} & \textbf{0.9008} \\ 
\midrule

\textbf{Dataset}              & \multicolumn{6}{c|}{derm7point: Clinical}                                                                                                                                                 & \multicolumn{6}{c}{derm7point: Dermoscopic}                                                                                                                                                                         \\ 
\#images (\#classes) & \multicolumn{6}{c|}{1,011 (5)}                                                                                                                                                                                        & \multicolumn{6}{c}{1,011 (5)}                                                                                                                                                                                        \\ 
Method               & \multicolumn{3}{c:}{ResNet-18}                                                                                      & \multicolumn{3}{c|}{ResNet-50}                                                                 & \multicolumn{3}{c:}{ResNet-18}                                                                                      & \multicolumn{3}{c}{ResNet-50}                                                                 \\ 
                     & \multicolumn{1}{c}{ACC$_{\mathrm{bal}}$}        & \multicolumn{1}{c}{F1-micro}        & \multicolumn{1}{c:}{F1-macro}        & \multicolumn{1}{c}{ACC$_{\mathrm{bal}}$}        & \multicolumn{1}{c}{F1-micro}        & F1-macro        & \multicolumn{1}{c}{ACC$_{\mathrm{bal}}$}        & \multicolumn{1}{c}{F1-micro}        & \multicolumn{1}{c:}{F1-macro}        & \multicolumn{1}{c}{ACC$_{\mathrm{bal}}$}        & \multicolumn{1}{c}{F1-micro}        & F1-macro        \\ \hdashline 
ERM                  & \multicolumn{1}{c}{42.08\%}          & \multicolumn{1}{c}{0.5297}          & \multicolumn{1}{c:}{0.3797}          & \multicolumn{1}{c}{42.15\%}          & \multicolumn{1}{c}{0.6485}          & 0.4328          & \multicolumn{1}{c}{54.79\%}          & \multicolumn{1}{c}{0.7030}          & \multicolumn{1}{c:}{0.5670}          & \multicolumn{1}{c}{55.46\%}          & \multicolumn{1}{c}{0.7574}          & 0.5819          \\ 
\mixup~               & \multicolumn{1}{c}{46.68\%}          & \multicolumn{1}{c}{0.5941}          & \multicolumn{1}{c:}{0.4392}          & \multicolumn{1}{c}{45.57\%}          & \multicolumn{1}{c}{0.6485}          & 0.4474          & \multicolumn{1}{c}{55.38\%}          & \multicolumn{1}{c}{0.7376}          & \multicolumn{1}{c:}{0.5683}          & \multicolumn{1}{c}{62.08\%}          & \multicolumn{1}{c}{\textbf{0.7772}} & 0.6419          \\ 
\nmixup~(2.4)        & \multicolumn{1}{c}{{\ul 47.82\%}}    & \multicolumn{1}{c}{{\ul 0.6782}}    & \multicolumn{1}{c:}{{\ul 0.4833}}    & \multicolumn{1}{c}{{\ul 46.63\%}}    & \multicolumn{1}{c}{0.6436}          & 0.4239          & \multicolumn{1}{c}{{\ul 55.88\%}}    & \multicolumn{1}{c}{{\ul 0.7525}}    & \multicolumn{1}{c:}{\textbf{0.5914}} & \multicolumn{1}{c}{\textbf{64.59\%}} & \multicolumn{1}{c}{0.7376}          & 0.6406          \\ 
\nmixup~(2.8)        & \multicolumn{1}{c}{\textbf{48.91\%}} & \multicolumn{1}{c}{0.6089}          & \multicolumn{1}{c:}{0.4496}          & \multicolumn{1}{c}{\textbf{48.36\%}} & \multicolumn{1}{c}{{\ul 0.6733}}    & \textbf{0.5122} & \multicolumn{1}{c}{\textbf{56.41\%}} & \multicolumn{1}{c}{\textbf{0.7574}} & \multicolumn{1}{c:}{{\ul 0.5700}}    & \multicolumn{1}{c}{{\ul 62.98\%}}    & \multicolumn{1}{c}{{\ul 0.7624}}    & {\ul 0.6552}    \\ 
\nmixup~(4.0)        & \multicolumn{1}{c}{46.93\%}          & \multicolumn{1}{c}{\textbf{0.7030}} & \multicolumn{1}{c:}{\textbf{0.4902}} & \multicolumn{1}{c}{45.95\%}          & \multicolumn{1}{c}{\textbf{0.6881}} & {\ul 0.4828}    & \multicolumn{1}{c}{55.45\%}          & \multicolumn{1}{c}{0.7178}          & \multicolumn{1}{c:}{0.5618}          & \multicolumn{1}{c}{62.58\%}          & \multicolumn{1}{c}{\textbf{0.7772}} & \textbf{0.6622} \\ 
\midrule

\textbf{Dataset}              & \multicolumn{6}{c|}{PH2}                                                                                                                                                                   & \multicolumn{6}{c}{MED-NODE}                                                                                                                                                                                        \\ 
\#images (\#classes) & \multicolumn{6}{c|}{200 (2)}                                                                                                                                                                                         & \multicolumn{6}{c}{170 (2)}                                                                                                                                                                                         \\ 
Method               & \multicolumn{3}{c:}{ResNet-18}                                                                                      & \multicolumn{3}{c|}{ResNet-50}                                                                 & \multicolumn{3}{c:}{ResNet-18}                                                                                      & \multicolumn{3}{c}{ResNet-50}                                                                 \\ 
                     & \multicolumn{1}{c}{ACC$_{\mathrm{bal}}$}        & \multicolumn{1}{c}{F1-micro}        & \multicolumn{1}{c:}{F1-macro}        & \multicolumn{1}{c}{ACC$_{\mathrm{bal}}$}        & \multicolumn{1}{c}{F1-micro}        & F1-macro        & \multicolumn{1}{c}{ACC$_{\mathrm{bal}}$}        & \multicolumn{1}{c}{F1-micro}        & \multicolumn{1}{c:}{F1-macro}        & \multicolumn{1}{c}{ACC$_{\mathrm{bal}}$}        & \multicolumn{1}{c}{F1-micro}        & F1-macro        \\ \hdashline 
ERM                  & \multicolumn{1}{c}{84.38\%}          & \multicolumn{1}{c}{0.8000}          & \multicolumn{1}{c:}{0.8438}          & \multicolumn{1}{c}{84.38\%}          & \multicolumn{1}{c}{{\ul 0.9000}}    & {\ul 0.8438}    & \multicolumn{1}{c}{75.00\%}          & \multicolumn{1}{c}{{\ul 0.7941}}    & \multicolumn{1}{c:}{0.7589}          & \multicolumn{1}{c}{74.64\%}          & \multicolumn{1}{c}{{\ul 0.7647}}    & 0.7509          \\ 
\mixup~               & \multicolumn{1}{c}{{\ul 85.94\%}}    & \multicolumn{1}{c}{{\ul 0.9250}}    & \multicolumn{1}{c:}{{\ul 0.8769}}    & \multicolumn{1}{c}{{\ul 85.94\%}}    & \multicolumn{1}{c}{0.8500}          & 0.8000          & \multicolumn{1}{c}{80.36\%}          & \multicolumn{1}{c}{{\ul 0.7941}}    & \multicolumn{1}{c:}{0.7925}          & \multicolumn{1}{c}{\textbf{81.79\%}} & \multicolumn{1}{c}{\textbf{0.8235}} & \textbf{0.8179} \\ 
\nmixup~(2.4)        & \multicolumn{1}{c}{{\ul 85.94\%}}    & \multicolumn{1}{c}{{\ul 0.9250}}    & \multicolumn{1}{c:}{{\ul 0.8769}}    & \multicolumn{1}{c}{\textbf{87.50\%}} & \multicolumn{1}{c}{\textbf{0.9500}} & \textbf{0.9134} & \multicolumn{1}{c}{79.29\%}          & \multicolumn{1}{c}{{\ul 0.7941}}    & \multicolumn{1}{c:}{0.7986}          & \multicolumn{1}{c}{{\ul 80.71\%}}    & \multicolumn{1}{c}{\textbf{0.8235}} & {\ul 0.8132}    \\ 
\nmixup~(2.8)        & \multicolumn{1}{c}{\textbf{96.88\%}} & \multicolumn{1}{c}{\textbf{0.9500}} & \multicolumn{1}{c:}{\textbf{0.9283}} & \multicolumn{1}{c}{\textbf{87.50\%}} & \multicolumn{1}{c}{\textbf{0.9500}} & \textbf{0.9134} & \multicolumn{1}{c}{\textbf{82.86\%}} & \multicolumn{1}{c}{\textbf{0.8235}} & \multicolumn{1}{c:}{\textbf{0.8211}} & \multicolumn{1}{c}{\textbf{81.79\%}} & \multicolumn{1}{c}{\textbf{0.8235}} & \textbf{0.8179} \\ 
\nmixup~(4.0)        & \multicolumn{1}{c}{{\ul 85.94\%}}    & \multicolumn{1}{c}{{\ul 0.9250}}    & \multicolumn{1}{c:}{{\ul 0.8769}}    & \multicolumn{1}{c}{\textbf{87.50\%}} & \multicolumn{1}{c}{\textbf{0.9500}} & \textbf{0.9134} & \multicolumn{1}{c}{{\ul 81.79\%}}    & \multicolumn{1}{c}{\textbf{0.8235}} & \multicolumn{1}{c:}{{\ul 0.8179}}    & \multicolumn{1}{c}{{\ul 80.71\%}}    & \multicolumn{1}{c}{\textbf{0.8235}} & {\ul 0.8132}    \\ 
\bottomrule
\end{tabular}%
}
\end{table*}

Due to space constraints, only the balanced accuracy values for SKIN are reported in the paper (Table~\ref{tab:skin_results_narrow}). Table~\ref{tab:skin_results} lists balanced accuracy (ACC$_{\mathrm{bal}}$) and micro- and macro-averaged F1 scores (F1-micro and F1-macro respectively) for all the models trained and evaluated on SKIN.



\end{document}